\documentclass{article}

\usepackage{arxiv}

\usepackage[utf8]{inputenc}
\usepackage[T1]{fontenc}
\usepackage{url}
\usepackage{booktabs}
\usepackage{amsfonts}
\usepackage{amsmath,amssymb,amsthm}
\usepackage{nicefrac}
\usepackage{microtype}
\usepackage{graphicx}
\usepackage{natbib}
\usepackage{doi}
\usepackage{algorithm}
\usepackage{algorithmic}
\usepackage{float}
\usepackage{placeins}
\usepackage{multirow}
\usepackage{subcaption}
\usepackage{xcolor}
\usepackage{listings}
\usepackage[capitalize]{cleveref}
\usepackage{hyperref}
\usepackage{orcidlink}

\newcommand{\eg}{\textit{e.g.}}
\newcommand{\ie}{\textit{i.e.}}

\newtheorem{theorem}{Theorem}
\newtheorem{proposition}[theorem]{Proposition}
\newtheorem{corollary}[theorem]{Corollary}

\title{UrbanAlign: Post-hoc Semantic Calibration for VLM-Human Preference Alignment}

\author{
Yecheng Zhang\textsuperscript{1}\thanks{Corresponding: \texttt{zhangyec23@mails.tsinghua.edu.cn} (Y.~Zhang), \texttt{rong.zhao.25@ucl.ac.uk} (R.~Zhao), \texttt{230238514@seu.edu.cn} (C.~Shi)}\,\orcidlink{0000-0002-4151-2737}\quad
Rong Zhao\textsuperscript{2}\footnotemark[1]\quad
Zhizhou Sha\textsuperscript{1}\quad
Yong Li\textsuperscript{3}\quad
Lei Wang\textsuperscript{4}\quad
Ce Hou\textsuperscript{3}\quad
Wen Ji\textsuperscript{5} \\[2pt]
\bfseries Hao Huang\textsuperscript{1}\quad
Yunshan Wan\textsuperscript{6}\quad
Jian Yu\textsuperscript{7}\quad
Junhao Xia\textsuperscript{1}\quad
Yuru Zhang\textsuperscript{8}\quad
Chunlei Shi\textsuperscript{9}\footnotemark[1] \\[8pt]
\mdseries\textsuperscript{1}Tsinghua University\quad
\textsuperscript{2}University College London\quad
\textsuperscript{3}Hong Kong University of Science and Technology \\[2pt]
\mdseries\textsuperscript{4}Peking University\quad
\textsuperscript{5}Southwest Jiaotong University\quad
\textsuperscript{6}Zhejiang University \\[2pt]
\mdseries\textsuperscript{7}University of Texas at Austin\quad
\textsuperscript{8}Renmin University of China\quad
\textsuperscript{9}Southeast University
}

\renewcommand{\shorttitle}{UrbanAlign: Post-hoc Semantic Calibration for VLM-Human Preference Alignment}

\hypersetup{
hidelinks,
pdftitle={UrbanAlign: Post-hoc Semantic Calibration for VLM-Human Perception Alignment},
pdfauthor={Yecheng Zhang, Rong Zhao, Zhizhou Sha, Yong Li, Lei Wang, Ce Hou, Wen Ji, Hao Huang, Yunshan Wan, Jian Yu, Junhao Xia, Yuru Zhang, Chunlei Shi},
pdfkeywords={Human Preference Alignment, Post-hoc Concept Bottleneck, Vision-Language Model, Urban Perception},
}

\begin{document}
\maketitle

\begin{abstract}
Vision-language models (VLMs) can describe urban scenes in rich detail, yet consistently fail to produce reliable human preference labels in domain-specific tasks such as safety assessment and aesthetic evaluation.
The standard fix, fine-tuning or RLHF, requires large-scale annotations and model retraining.
We ask a different question: \emph{can a frozen VLM be aligned with human preferences without modifying any weights?}
Our key insight is that VLMs are strong concept extractors but poor decision calibrators.
We propose a three-stage post-hoc pipeline that exploits this asymmetry: (i)~interpretable evaluation dimensions are automatically mined from consensus exemplars; (ii)~an Observer--Debater--Judge chain extracts robust concept scores from the frozen VLM; and (iii)~locally-weighted ridge regression on a hybrid manifold calibrates these scores to human ratings.
Applied as \textbf{UrbanAlign} on Place Pulse~2.0, the framework reaches 72.2\% accuracy ($\kappa{=}0.45$) across six perception categories, outperforming all baselines by +11.0\,pp and zero-shot VLM by +15.5\,pp, with full interpretability and zero weight modification.
\end{abstract}


\section{Introduction}
\label{sec:intro}

Large vision-language models (VLMs) excel at describing visual scenes but consistently fail to produce reliable preference labels for domain-specific tasks.
This \emph{VLM--human preference alignment gap} appears across domains, from urban perception~\citep{zhang2025genai,mushkani2025dovlms} to aesthetic quality~\citep{murray2012ava,wu2024qalign}, and the standard remedy is model adaptation: fine-tuning, LoRA, or RLHF~\citep{ouyang2022training}, all of which modify model weights and require domain-specific training data and non-trivial compute.
We ask a different question: \emph{can a frozen VLM be aligned with human preferences in a new domain, without modifying any model weights?}

The key observation is that VLMs are already strong \emph{concept extractors} that can identify rich visual elements in a scene (e.g., building facades, vegetation, and street infrastructure), but poor \emph{decision calibrators}: their mapping from these features to discrete comparison labels is poorly aligned with human judgement boundaries (\cref{fig:comparison}).
This mirrors the finding of Concept Bottleneck Models (CBMs)~\citep{koh2020concept} that routing predictions through interpretable concept scores and then a calibrated linear layer dramatically outperforms end-to-end classification.
Such a bottleneck can be retrofitted \emph{post-hoc} onto any frozen backbone~\citep{yuksekgonul2023posthoc}, and recent work further shows that concepts can be discovered directly from VLMs without manual labels~\citep{oikarinen2023labelfree,yang2023language}.

A na\"{i}ve alternative is to prompt the VLM as a single-shot scorer, but this suffers from well-documented biases: position effects, anchoring, and inconsistent self-evaluation~\citep{wu2024qalign,liu2024pairs}.
When asked to compare two scenes directly, a single VLM call conflates perception and judgement: the model must simultaneously identify relevant visual cues, weigh their relative importance, and produce a discrete label, all within a single forward pass.
This end-to-end compression discards the structured reasoning that human raters implicitly perform when evaluating complex, multi-dimensional scenes.
Multi-agent debate protocols decompose complex judgements into complementary reasoning steps and have been shown to mitigate such biases through structured disagreement~\citep{du2024improving,wang2023selfconsistency}, yet these techniques have primarily targeted factual reasoning and text generation; their potential for domain-specific visual preference alignment remains unexplored.

This principle is operationalised as a pipeline that keeps the VLM entirely frozen, with three tightly coupled stages (\cref{fig:framework}):
\textbf{(i)}~\emph{concept mining}, where interpretable evaluation dimensions are discovered by the VLM from consensus exemplars;
\textbf{(ii)}~\emph{structured scoring}, where an Observer--Debater--Judge multi-agent chain extracts robust continuous concept scores from the frozen VLM;
\textbf{(iii)}~\emph{geometric calibration}, where locally-weighted ridge regression (LWRR)~\citep{cleveland1979robust} on a hybrid visual--semantic manifold aligns concept scores to human ratings.
An end-to-end optimization loop unifies these stages, automatically selecting the dimension set that maximises calibrated accuracy via temperature-scheduled trials.
We instantiate this framework as \textbf{UrbanAlign} on Place Pulse~2.0~\citep{salesses2013collaborative}, a large-scale pairwise comparison dataset covering six categories of urban perception, where subjective judgements~\citep{wilson1982broken,jacobs1961death} defy the feature-to-label mappings of existing approaches~\citep{dubey2016deep,naik2014streetscore}.
UrbanAlign substantially outperforms both supervised baselines and zero-shot VLM scoring across all six categories, with full dimension-level interpretability and zero model-weight modification, demonstrating that a frozen VLM already contains sufficient perceptual knowledge for human-aligned urban assessment.

\smallskip\noindent\textbf{Contributions.}
\begin{enumerate}
\item \textbf{End-to-end concept mining.} We design a loop that automatically discovers interpretable evaluation dimensions from consensus exemplars and refines them via a temperature-scheduled search, forming a concept bottleneck for perception prediction.
\item \textbf{Multi-agent structured scoring.} We introduce an Observer--Debater--Judge deliberation chain that extracts robust continuous concept scores from a frozen VLM, reducing single-agent bias through structured disagreement.
\item \textbf{Local geometric calibration.} We propose locally-weighted ridge regression on a hybrid visual--semantic manifold that calibrates concept scores against human ratings with per-sample interpretability.
\end{enumerate}

\begin{figure}[!ht]
\centering
\includegraphics[width=\textwidth]{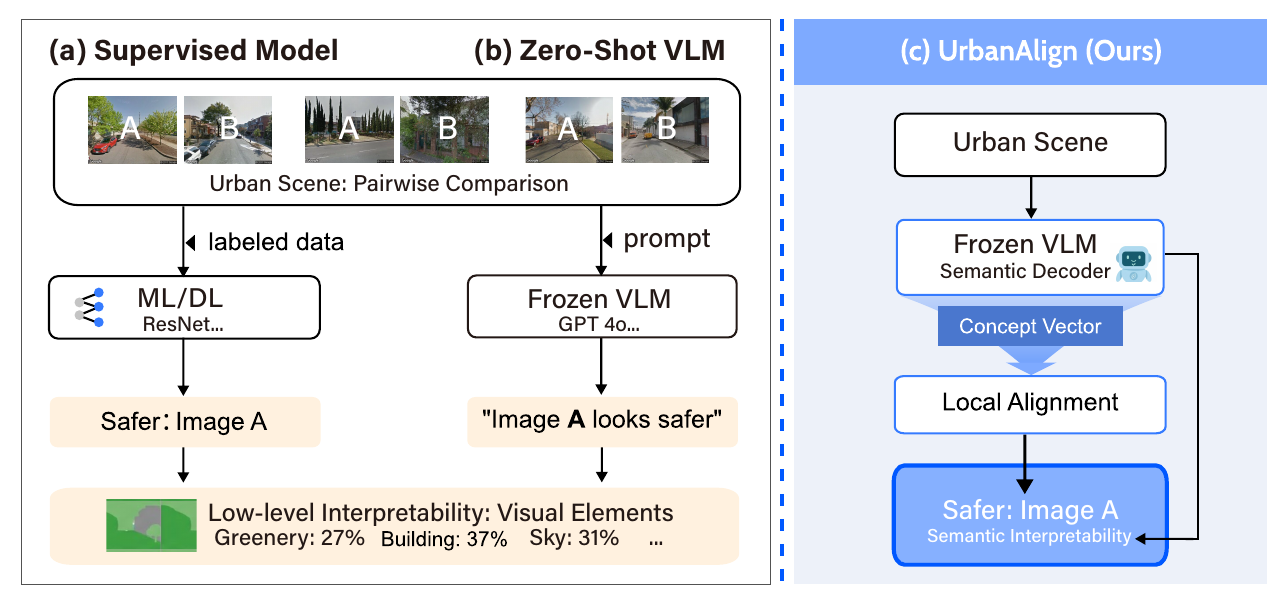}
\caption{Three approaches to urban perception: \textbf{(a)}~supervised baselines (e.g., Siamese networks, segmentation regression) require labelled data and lack interpretable intermediates; \textbf{(b)}~zero-shot VLM prompting is training-free but produces opaque, uncalibrated judgements; \textbf{(c)}~UrbanAlign routes predictions through VLM-discovered semantic dimensions for interpretable, locally calibrated comparison.}
\label{fig:comparison}
\end{figure}

\section{Related Work}
\label{sec:related}

\noindent\textbf{Urban perception and preference data.}
Pairwise comparison has a long history in psychometrics~\citep{thurstone1927law,bradley1952rank}.
Place Pulse collects over one million pairwise comparisons of street-view images across six perceptual dimensions~\citep{salesses2013collaborative}.
Siamese networks~\citep{dubey2016deep}, multi-task learning~\citep{zhang2018measuring}, and continuous score regression~\citep{naik2014streetscore} have been trained on these labels, but map directly from pixels to abstract percepts without interpretable mid-level representations.
Cross-cultural evaluations further show that such models transfer poorly to non-Western cities~\citep{pp2chinese2025}, motivating cost-effective local calibration rather than large-scale re-annotation.
Pairwise preference datasets now extend to 10-dimension urban perception~\citep{quintana2025specs} and photographic aesthetics~\citep{murray2012ava}, paralleling the learning-to-rank paradigm~\citep{joachims2002optimizing}, yet the resulting models remain black-box rankers.
Foundation models have been used to generate synthetic training corpora~\citep{ye2022zerogen,meng2023tuning}, but these operate on unimodal text with discrete labels and do not address the multi-dimensional, subjective nature of visual preference.
UrbanAlign instead synthesises \emph{visual} pairwise data via structured dimension scoring and post-hoc geometric calibration, requiring no weight modification.

\smallskip\noindent\textbf{VLM-based scoring and multi-agent reasoning.}
Systematic evaluation reveals that VLMs capture broad perceptual trends but oversimplify multi-dimensional human judgements~\citep{zhang2025genai}; benchmarking on urban perception confirms stronger alignment on objective properties (\eg, building density) than subjective ones (\eg, safety, boringness)~\citep{mushkani2025dovlms}.
Contrastive pretraining on satellite imagery~\citep{yan2024urbanclip} and zero-shot urban function inference~\citep{huang2024zeroshot} further demonstrate both the power and limits of VLMs for urban tasks; these studies position VLMs as end-to-end classifiers, whereas we reposition them as \emph{structured semantic decoders} whose outputs require post-hoc calibration.
Pairwise evaluation aligns better with human opinions than pointwise scoring~\citep{wu2024qalign,liu2024pairs}, yet single-shot scoring still suffers from position effects and anchoring biases.
Self-consistency sampling~\citep{wang2023selfconsistency} and multi-agent debate~\citep{du2024improving,liang2024encouraging,liu2025dmad} mitigate these issues for factual QA and mathematical reasoning, but applying structured deliberation to continuous visual scoring with domain-specific semantics remains unexplored.
UrbanAlign bridges this gap with an Observer--Debater--Judge chain that separates observation, argumentation, and judgement into distinct reasoning stages for pairwise dimension scoring.

\smallskip\noindent\textbf{VLM preference alignment.}
Recent VLM alignment methods can be grouped by whether they modify model weights.
RLHF-based approaches train reward models and fine-tune the VLM via policy optimisation~\citep{ouyang2022training}; extensions augment this pipeline with factual grounding~\citep{sun2023factually} and vision-guided reinforcement that eliminates human annotation entirely~\citep{zhan2025visionr1}.
Robust visual reward models leverage auxiliary textual preference data to improve reward quality under scarce visual annotations~\citep{wang2024rovrm}, while implicit preference mining from user-generated content broadens the training signal beyond curated labels~\citep{tan2025implicit}.
All these methods require modifying VLM weights.
A parallel line calibrates VLM confidence post-hoc via semantic perturbation at the token level~\citep{zhao2025calibration}, but targets per-query confidence rather than domain-specific preference alignment.
UrbanAlign occupies a distinct position: it requires \emph{zero weight modification}, instead routing predictions through an external concept bottleneck with local geometric calibration---combining the interpretability of CBMs with the locality of retrieval-augmented prediction.

\smallskip\noindent\textbf{Concept bottleneck models and post-hoc calibration.}
Modern neural networks are known to be poorly calibrated~\citep{guo2017calibration}; temperature scaling and Platt scaling partially address this, but domain-specific preference alignment requires richer structural remedies.
Concept Bottleneck Models (CBMs) route predictions through interpretable concept scores, enabling human-auditable reasoning and concept-level intervention~\citep{koh2020concept}.
Post-hoc retrofitting onto frozen backbones~\citep{yuksekgonul2023posthoc} and automated concept discovery from VLMs~\citep{oikarinen2023labelfree,yang2023language} eliminate manual concept engineering, making CBMs applicable without task-specific annotation of concept labels.
Our semantic dimensions serve as an analogous interpretable bottleneck, but operate \emph{externally} on VLM-synthesised data rather than as an internal network layer~\citep{srivastava2024vlgcbm,liu2025hybridcbm}: the VLM generates continuous concept scores, and a separate calibration step maps them to human-aligned predictions.

The calibration mechanism draws on locally-weighted regression~\citep{cleveland1979robust,atkeson1997locally,ruppert1994multivariate} and can be viewed as a visual analogue of retrieval-augmented generation (RAG)~\citep{lewis2020retrieval}: for each query pair, the system retrieves the $K$ most similar reference pairs from the hybrid manifold and locally fits a calibration function, adapting dimension weights to the manifold neighbourhood.
Unlike global linear probes or fixed concept weights, this local fitting allows the relative importance of dimensions to vary smoothly across the perceptual manifold, capturing the context-dependent nature of human preference.

\newpage
\section{Method}
\label{sec:method}

\begin{figure*}[t]
\centering
\includegraphics[width=\textwidth]{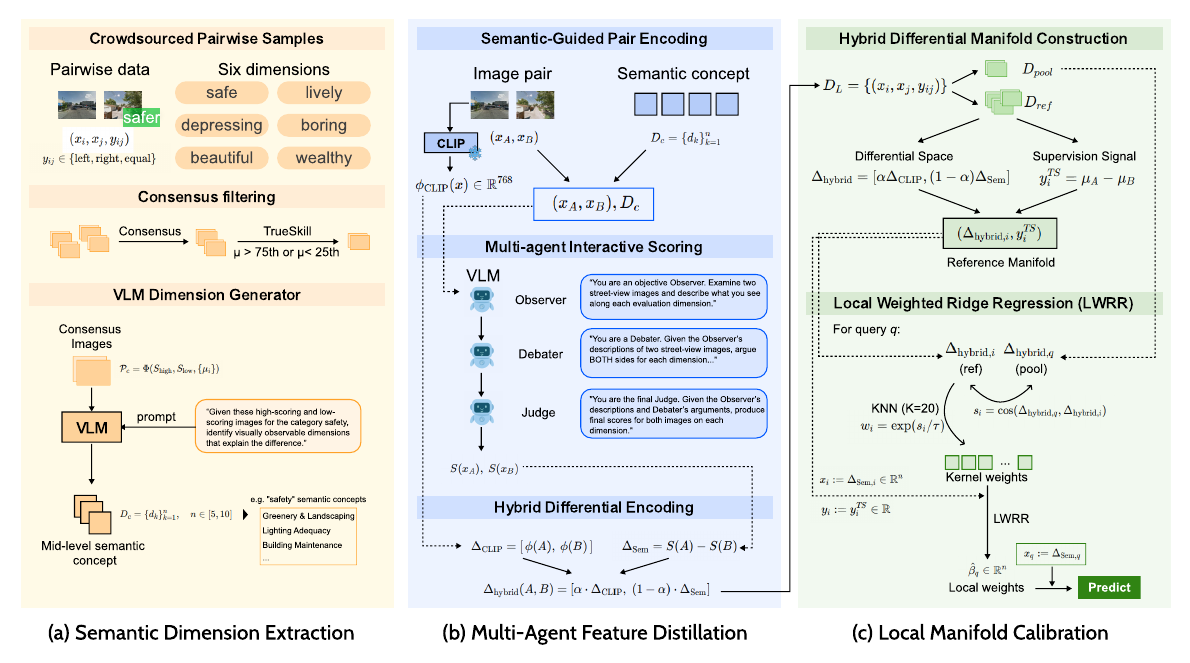}
\caption{UrbanAlign framework. Stage~1 mines semantic dimensions from consensus exemplars; Stage~2 distils dimension scores via Observer--Debater--Judge; Stage~3 calibrates via LWRR on the hybrid differential manifold. $\mathcal{D}_{\mathrm{ref}} \cap \mathcal{D}_{\mathrm{val}} \cap \mathcal{D}_{\mathrm{test}} = \varnothing$.}
\label{fig:framework}
\end{figure*}

\subsection{Problem Formulation}
\label{sec:formulation}

Given an urban image set $\mathcal{X}{=}\{x_1,\ldots,x_N\}$ and a perception category $d$ (\eg, wealthy), we aim to predict human pairwise preferences:
\begin{equation}
y_{ij} =
\begin{cases}
\text{left}  & \text{if } x_i \succ x_j \text{ on dimension } d,\\
\text{right} & \text{if } x_i \prec x_j,\\
\text{equal} & \text{if } x_i \approx x_j.
\end{cases}
\label{eq:pairwise}
\end{equation}
The labelled pool $\mathcal{D}_{L}{=}\{(x_i,x_j,y_{ij})\}$ contains crowdsourced pairwise comparisons (from Place Pulse~2.0).
We split $\mathcal{D}_{L}$ into a \emph{reference set} $\mathcal{D}_{\mathrm{ref}}$ (for LWRR calibration), a \emph{validation set} $\mathcal{D}_{\mathrm{val}}$ (for dimension and hyperparameter search), and a \emph{test set} $\mathcal{D}_{\mathrm{test}}$ (for final evaluation), with strict isolation: $\mathcal{D}_{\mathrm{ref}} \cap \mathcal{D}_{\mathrm{val}} \cap \mathcal{D}_{\mathrm{test}} = \varnothing$.

\noindent\textbf{Notation.}
$\phi{:}\,\mathcal{X}{\to}\mathbb{R}^{768}$ denotes a pre-trained CLIP encoder (ViT-L/14)~\citep{radford2021learning}.
TrueSkill~\citep{herbrich2006trueskill} converts pairwise comparisons to continuous ratings $\mu_i$ per image.

\subsection{Framework Overview}
\label{sec:overview}

The pipeline consists of three stages unified by an end-to-end optimization loop (\cref{fig:framework}).
\emph{Concept mining} (\cref{sec:stage1}) discovers interpretable evaluation dimensions from a handful of consensus exemplars and optimises them via a temperature-scheduled search.
\emph{Multi-agent structured scoring} (\cref{sec:stage2}) extracts continuous dimension scores from a frozen VLM through Observer--Debater--Judge deliberation, producing hybrid visual--semantic feature vectors.
\emph{Local manifold calibration} (\cref{sec:stage3}) aligns these scores to human ratings via locally-weighted ridge regression on the hybrid differential manifold.
Data isolation ($\mathcal{D}_{\mathrm{ref}} \cap \mathcal{D}_{\mathrm{val}} \cap \mathcal{D}_{\mathrm{test}} = \varnothing$) ensures that calibration, validation, and evaluation data never overlap.

\subsection{Concept Mining and Dimension Optimization (\cref{fig:framework}a)}
\label{sec:stage1}

\noindent\textbf{Motivation.}
Instead of asking a VLM ``Which image looks more wealthy?'' (a poorly calibrated end-to-end query), we first decompose the abstract percept into interpretable, continuously scorable sub-dimensions.
This is analogous to defining the concept set in a CBM~\citep{koh2020concept}: the dimensions form an interpretable bottleneck through which all subsequent predictions must pass.

\smallskip\noindent\textbf{TrueSkill consensus sampling.}
We convert pairwise comparisons to continuous scores via TrueSkill~\citep{herbrich2006trueskill}.
We sample \emph{high-consensus} ($\mu{>}\tau_{h},\;\sigma{<}\sigma_{\mathrm{med}}$) and \emph{low-consensus} ($\mu{<}\tau_{l},\;\sigma{<}\sigma_{\mathrm{med}}$) exemplars, where $\tau_{h}$ and $\tau_{l}$ are the 75th and 25th percentiles.
A small set per group forms the consensus set $\mathcal{S}_{\mathrm{consensus}}$.

\smallskip\noindent\textbf{Dimension extraction.}
A structured prompt presents the high/low exemplars with their TrueSkill scores and PCA-compressed CLIP embeddings.
The VLM outputs $n\in[5,10]$ dimensions in JSON, each containing a name, description, and high/low indicators (exact schema and prompts in supplementary \cref{sec:prompts}).
For the \emph{wealthy} category, the extracted dimensions are:
\emph{Fa\c{c}ade Quality}, \emph{Vegetation Maintenance}, \emph{Pavement Integrity}, \emph{Vehicle Quality}, \emph{Building Modernity}, \emph{Infrastructure Condition}, \emph{Street Cleanliness}, and \emph{Lighting Quality}.

These dimensions satisfy three requirements: \emph{visual observability} (scorable from street-view images), \emph{measurability} (1--10 continuous scale), and \emph{universality} (applicable across cities).

\smallskip\noindent\textbf{End-to-end dimension optimization.}
The dimensions discovered above depend on the VLM's sampling temperature and the particular consensus exemplars presented.
We close the loop with an automated search over dimension sets, optimising each perception category independently.
Let $\mathcal{D}_c^{(t)}$ denote the dimension set generated at trial $t$ for category $c$.
Each trial runs the full pipeline (dimension extraction $\to$ multi-agent scoring $\to$ LWRR calibration) on the held-out validation split $\mathcal{D}_{\mathrm{val}}$.
The per-category optimal set is:
\begin{equation}
\mathcal{D}_c^{*} = \arg\max_{t \in \{1,\ldots,T\}} \;\mathrm{Acc}_c\bigl(\text{LWRR}\bigl(\text{Score}(\mathcal{D}_{\mathrm{val}}^{c},\, \mathcal{D}_c^{(t)}),\, \mathcal{D}_{\mathrm{ref}}^{c}\bigr)\bigr).
\label{eq:e2e}
\end{equation}
The search proceeds in two phases:
\emph{Explore} (high $\tau_{\mathrm{gen}}$): diverse, independently generated dimension sets; each category accumulates its best.
\emph{Converge} (low $\tau_{\mathrm{gen}}$): per-category best dimensions are preserved except for 1--2 targeted replacements (\emph{mutation mode}), enabling local refinement.
Since categories may peak at different trials, the final output \emph{assembles} per-category optima: $\{\mathcal{D}_c^{*}\}_{c=1}^{C}$.

\subsection{Multi-Agent Structured Scoring (\cref{fig:framework}b)}
\label{sec:stage2}

\noindent\textbf{From classification to feature extraction.}
Rather than querying a VLM for a one-shot discrete label, we extract a \emph{hybrid feature vector}:
\begin{equation}
h(x) = [\phi_{\text{CLIP}}(x),\; S(x)],
\label{eq:hybrid}
\end{equation}
where $S(x)\in[1,10]^{n}$ is the vector of semantic dimension scores produced by the VLM, fusing low-level visual patterns with high-level domain semantics analogously to a post-hoc CBM~\citep{yuksekgonul2023posthoc}.

\smallskip\noindent\textbf{Observer--Debater--Judge deliberation.}
Inspired by multi-agent debate~\citep{du2024improving,liang2024encouraging} and self-consistency~\citep{wang2023selfconsistency}, we decompose reasoning into three sequentially chained agents:

\begin{itemize}
\item \textbf{Observer}: describes visually observable details per dimension, without judgement, thereby suppressing confirmation bias.
\item \textbf{Debater}: argues both for HIGH and LOW scores on each dimension, exploring opposing perspectives.
\item \textbf{Judge}: synthesises descriptions and arguments to produce final scores $S(x)\in[1,10]^{n}$.
\end{itemize}

\noindent\textbf{Variance reduction.}
The three-agent chain reduces dimension score variance from $\sigma_S^2$ to at most $\sigma_S^2/(1{+}2(1{-}\rho))$, where $\rho\in[0,1]$ is inter-agent correlation; when agents reason independently ($\rho{\to}0$), variance drops by up to $3\times$ (proof in \cref{sec:theory_multiagent}).

\smallskip\noindent\textbf{Intensity-based winner determination.}
For each pair $(x_A, x_B)$, we also record the \emph{intensity}: $I{=}|S(x_A) - S(x_B)|$ summed across dimensions.
Pairs with intensity below a threshold $\sigma_I$ are labelled ``equal''; otherwise the VLM's predicted winner is retained.

\subsection{Local Manifold Calibration (\cref{fig:framework}c)}
\label{sec:stage3}

This stage is the core algorithmic contribution: a post-hoc calibration layer analogous to the linear predictor in a CBM~\citep{koh2020concept}, but \emph{locally adaptive} to account for heterogeneous perception patterns across the visual-semantic manifold.

\subsubsection{Hybrid Differential Space.}

CLIP, pre-trained on web-scale image-text pairs, provides strong general features but exhibits distribution shift on street-view scenes.
We augment it with semantic dimension scores to form the hybrid differential encoding:
\begin{equation}
\Delta_{\text{hybrid}}(A,B)=\bigl[\alpha\cdot\overline{\Delta}_{\text{CLIP}}(A,B),\;(1{-}\alpha)\cdot\overline{\Delta}_{\text{Sem}}(A,B)\bigr],
\label{eq:hybrid_diff}
\end{equation}
where $\overline{\Delta}_{\text{CLIP}}$ uses L2-normalised CLIP embeddings, $\Delta_{\text{Sem}}{=}S(A){-}S(B)$ is normalised to $[0,1]$ by dividing by 10, and $\alpha\in[0,1]$ balances visual and semantic components.

\subsubsection{LWRR Algorithm.}
Our calibration extends locally-weighted regression~\citep{cleveland1979robust,atkeson1997locally,ruppert1994multivariate} to the hybrid differential space defined above.

\noindent\textbf{Step~1: Build reference manifold.}
From $\mathcal{D}_{\mathrm{ref}}$, compute hybrid differential vectors and TrueSkill score differences $y_i^{\text{TS}}{=}\mu_A{-}\mu_B$.
Mirror augmentation doubles the reference set: for each $(A,B,y)$, add $(B,A,-y)$.

\noindent\textbf{Step~2: Neighbourhood search.}
For query $\Delta_q^{\mathrm{hybrid}} \in \mathcal{D}_{\mathrm{test}}$, compute cosine similarities to all reference points and select the $K$ nearest neighbours $\mathcal{N}_K$.

\noindent\textbf{Step~3: Kernel weighting.}
$w_i=\exp(s_i/\tau)$ for $i\in\mathcal{N}_K$, where $s_i$ is cosine similarity and $\tau$ is kernel bandwidth.

\noindent\textbf{Step~4: Local ridge regression.}
Solve for local dimension weights:
\begin{equation}
\hat{\mathbf{w}} = \arg\min_{\mathbf{w}} \sum_{i\in\mathcal{N}_K} w_i\bigl(y_i^{\text{TS}}-\mathbf{w}^{\!\top}\Delta_{\text{Sem},i}\bigr)^{2} + \lambda\|\mathbf{w}\|^{2},
\label{eq:lwrr}
\end{equation}
with closed-form solution $\hat{\mathbf{w}}=(X^{\!\top}WX+\lambda I)^{-1}X^{\!\top}Wy$, where $W{=}\mathrm{diag}(w_1,\ldots,w_K)$ and $X$ stacks the semantic differentials.

\noindent\textbf{Step~5: Prediction.}
The calibrated score difference $\hat{\delta}_q = \hat{\mathbf{w}}^{\!\top}\Delta_{\text{Sem}}^{q}$; a local goodness-of-fit statistic provides per-prediction interpretability by measuring how much TrueSkill variance is explainable by the mid-level dimensions in each neighbourhood.

\noindent\textbf{Step~6: Statistical re-inference.}
\begin{equation}
\hat{y} =
\begin{cases}
\text{equal} & \text{if } |\hat{\delta}|<\varepsilon \text{ or } c_{\mathrm{eq}}>\theta, \\
\text{left}  & \text{if } \hat{\delta}>0, \\
\text{right} & \text{otherwise},
\end{cases}
\label{eq:reinfer}
\end{equation}
where $\varepsilon$ is a score-difference threshold and $c_{\mathrm{eq}}{=}|\{i\in\mathcal{N}_K: l_i{=}\text{equal}\}|/K$ measures local equal-consensus.

\begin{theorem}[Local Calibration Bound]\label{thm:calibration_main}
(Adapted from classical locally-weighted regression~\citep{cleveland1979robust,caponnetto2007optimal} to the hybrid VLM-concept manifold.)
Under the local model $y_i^{\mathrm{TS}} = w^{*\!}(q)^{\!\top} \Delta_{\mathrm{Sem},i} + \epsilon_i$ within $\mathcal{N}_K(q)$, with zero-mean noise of variance~$\sigma^2$, the LWRR estimator satisfies:
\begin{equation}
\mathbb{E}\bigl[\|\hat{w} - w^{*\!}(q)\|^2\bigr]
\;\leq\;
\frac{\sigma^2\,\mathrm{tr}(H_\lambda^{-1}X^{\!\top}W^2\!X\,H_\lambda^{-1})}{K}
\;+\;
\lambda^2\|w^{*}\|^2,
\label{eq:lwrr_bound_main}
\end{equation}
where $H_\lambda = X^{\!\top}\!W\!X + \lambda I$.
At optimal $\lambda^{*} \!=\! \sigma\sqrt{n/K}/\|w^{*}\|$, this simplifies to $O(\sigma\|w^{*}\|\sqrt{n/K})$.
The key implication for our setting is that with $n{=}5$--$8$ concept dimensions and $K{=}20$ neighbours, the system is well-overdetermined ($K \gg n$), ensuring small estimation variance; in contrast, a global model incurs irreducible bias from manifold heterogeneity (\cref{thm:local_global}).
\emph{Proof in \cref{sec:theory_calibration}.}
\end{theorem}

\noindent\textbf{Pipeline accuracy.}
The following decomposition provides intuitive guidance for why each pipeline stage is necessary, though it is approximate rather than a formal bound:
\begin{equation}
\mathrm{Acc} \;\approx\; \Phi\!\left(\frac{\|w^{*}\|\,\bar{\delta}}{\sqrt{\sigma_{\mathrm{LWRR}}^{2} + \sigma_{S}^{2}/\eta_{M}}}\right),
\label{eq:accuracy}
\end{equation}
where $\bar{\delta}$ is the mean signal strength (maximised by concept mining), $\sigma_S^2$ is dimension-score variance (reduced by multi-agent scoring with $\eta_M{=}1{+}2(1{-}\rho)$), and $\sigma_{\mathrm{LWRR}}^2{=}O(\sigma^2 n/K)$ is calibration error (minimised by LWRR).
Each pipeline stage tightens a specific term: Stage~1 maximises $\bar{\delta}$, Stage~2 minimises $\sigma_S^2/\eta_M$, and Stage~3 minimises $\sigma_{\mathrm{LWRR}}^2$; a formal end-to-end bound is given in \cref{cor:e2e}.
Full derivation in \cref{sec:theory}.

\smallskip\noindent\textbf{Complete algorithm.}
The full pseudocode is provided in \cref{sec:algorithm}.
Additional theoretical results (local vs.\ global calibration, hybrid space complementarity, E2E convergence) are in \cref{sec:theory}.

\newpage
\section{Experiments and Results}
\label{sec:exp}

\subsection{Setup}

\noindent\textbf{Dataset.}
Place Pulse 2.0~\citep{salesses2013collaborative}: 110,688 street-view images from 56 cities with 1,169,078 pairwise comparisons across six perception categories (safety, lively, beautiful, wealthy, depressing, boring).
We retain all pairs with ${\geq}2$ independent human votes (449--1,412 per category).
Results span \textbf{all six categories}; detailed ablation on \emph{wealthy}.

\smallskip\noindent\textbf{Data protocol.}
Each category is partitioned into $\mathcal{D}_{\mathrm{ref}}$ (60\%), $\mathcal{D}_{\mathrm{val}}$ (20\%), and $\mathcal{D}_{\mathrm{test}}$ (20\%); $\mathcal{D}_{\mathrm{ref}} \cap \mathcal{D}_{\mathrm{val}} \cap \mathcal{D}_{\mathrm{test}} = \varnothing$ (stratified split, seed\,=\,42).
$\mathcal{D}_{\mathrm{ref}}$ provides TrueSkill ratings for LWRR calibration; $\mathcal{D}_{\mathrm{val}}$ serves dimension search (\cref{eq:e2e}) and hyperparameter selection; $\mathcal{D}_{\mathrm{test}}$ is held out for final evaluation.
After excluding human ``equal'' labels, 83--260 test pairs per category remain.

\smallskip\noindent\textbf{Implementation.}
CLIP ViT-L/14 (768-d) for visual features; PCA to 8D for consensus prompts.
GPT-4o~\citep{openai2024gpt4o} as the default VLM; backbone-agnostic design is verified with Qwen2.5-VL-72B~\citep{bai2025qwen25vl} (\cref{sec:sensitivity}).
TrueSkill priors: $\mu{=}25$, $\sigma{=}8.33$, draw probability 0.10; 5 images per consensus group (10 total).
Stage~2: Observer $\tau{=}0.3$, Debater $\tau{=}0.5$, Judge $\tau{=}0.1$; single-shot modes $\tau{=}0.0$.
Stage~3 defaults: $K{=}20$, $\tau_{\mathrm{kernel}}{=}1.0$, $\lambda{=}1.0$, $\varepsilon{=}0.8$, $\theta{=}0.6$, $\alpha{=}0.3$.
E2E dimension search: $T{=}15$ trials with patience-based early stopping after 5 non-improving trials; explore phase $\tau_{\mathrm{gen}}{=}0.85{\to}1.0$, converge phase $\tau_{\mathrm{gen}}{=}0.7{\to}0.5$.
Total VLM cost: {\raise.17ex\hbox{$\scriptstyle\sim$}}\$68 for the experimental dataset (4,819 pairs, Mode~4); wall-clock time ${\sim}$2--4\,h with concurrent API access, no GPU required (\cref{sec:discussion}).
The main results (\cref{tab:main}) use \emph{per-category} optimised hyperparameters from a combined random search ($N{=}50{,}000$ configurations on $\mathcal{D}_{\mathrm{val}}$; \cref{sec:sensitivity_full}), as the optimal $K$ (10--50) and $\alpha$ (0.2--0.7) vary substantially across categories.

\smallskip\noindent\textbf{Baselines.}
\emph{ResNet-50 Siamese}: twin-tower network on frozen ResNet-50 features (2048-d $\to$ 512 $\to$ 128-d projection), concatenating $[\mathbf{z}_L,\mathbf{z}_R,|\mathbf{z}_L{-}\mathbf{z}_R|]$ for 3-way classification~\citep{koch2015siamese};
\emph{CLIP Siamese}: same twin-tower architecture on frozen CLIP ViT-L/14 features (768-d), representing a modern visual backbone;
\emph{Seg.\ + CLIP Regression}: SegFormer semantic segmentation (19 Cityscapes classes $\to$ 7 urban-element groups) pixel proportions concatenated with PCA-reduced CLIP features, difference vectors classified by Gradient Boosting~\citep{chen2018deeplabv3plus};
\emph{Zero-shot VLM}: GPT-4o with a forced binary-choice prompt ``Which looks more \texttt{<category>}?'' (no ``equal'' option), without semantic dimensions or multi-agent reasoning.
All supervised baselines are trained on the reference split with mirror augmentation (swapping left/right with flipped labels) and evaluated on the disjoint test split.

\smallskip\noindent\textbf{Metrics.}
Accuracy (Acc), Cohen's $\kappa$ (chance-corrected agreement), F1 (macro).
We report metrics both including and excluding pairs where the human label is ``equal'' (as these are inherently ambiguous).

\subsection{Main Results}

\Cref{tab:main} and \cref{fig:comparison} present the main comparison; \cref{fig:qualitative} provides per-pair qualitative evidence.
Key findings:

(1)~UrbanAlign (Mode~4 + LWRR) achieves 72.2\% accuracy ($\kappa{=}0.45$) on average across all six categories, with 81.6\% accuracy ($\kappa{=}0.64$) on \emph{safety} perception being the strongest.
This represents a \textbf{+11.0\,pp} gain over the strongest same-feature control (Global Ridge on $\Delta_{\mathrm{sem}}$: 61.2\%) and +15.5\,pp over zero-shot VLM (56.7\%).

(2)~Independent baselines cluster in 51--58\%; same-feature controls (BT Logistic, Global Ridge) trained on $\Delta_{\mathrm{sem}}$ plateau at 61.2\%, confirming that local calibration---not feature quality alone---is the key differentiator.
(3)~LWRR calibration provides +16.3\,pp average boost (55.9\%$\to$72.2\%), with the largest effect on \emph{boring} (+31.3\,pp) and \emph{depressing} (+25.3\,pp), where raw predictions are near chance.
LWRR succeeds because it re-weights dimensions \emph{locally}: a pair dominated by lighting cues receives different weights from one dominated by vegetation, capturing the context-dependent nature of human perception~\citep{lynch1960image}.
(4)~The gap between \emph{safety} (81.6\%) and \emph{boring} (70.2\%) reflects perceptual grounding: safety correlates with concrete cues~\citep{wilson1982broken,jeffery1971cpted}, while boringness depends on subtler absence-of-stimulation signals.
(5)~Human split-half agreement on Place Pulse ranges from 78--87\% across categories (\cref{sec:reliability}); on \emph{safety}, UrbanAlign matches this ceiling (81.8\% vs.\ 81.6\%), while on other categories a 9--18\,pp gap remains.
Bootstrap 95\% CIs for UrbanAlign accuracy are $\pm$5--10\,pp, reflecting moderate test-set sizes (83--260 pairs per category); all lower bounds exceed the 56.7\% zero-shot baseline.

\begin{table}[t]
\centering
\caption{Comparison on Place Pulse 2.0 across all six perception categories. All values reported as excl-equal\,(incl-equal). UrbanAlign outperforms all baselines including same-feature controls (72.2\% vs.\ 61.2\% best baseline).}
\label{tab:main}
\small
\resizebox{\columnwidth}{!}{%
\begin{tabular}{l c c c c c c}
\toprule
 & \multicolumn{2}{c}{Safety} & \multicolumn{2}{c}{Beautiful} & \multicolumn{2}{c}{Lively} \\
\cmidrule(lr){2-3} \cmidrule(lr){4-5} \cmidrule(lr){6-7}
Method & Acc.\,(\%) & $\kappa$ & Acc.\,(\%) & $\kappa$ & Acc.\,(\%) & $\kappa$ \\
\midrule
ResNet50 Siamese~\citep{koch2015siamese} & 54.8\,(54.0) & 0.16\,(0.18) & 58.6\,(56.0) & 0.18\,(0.19) & 45.7\,(42.5) & $-$0.07\,($-$0.06) \\
CLIP Siamese~\citep{radford2021learning} & 66.7\,(64.4) & 0.34\,(0.31) & 61.4\,(57.3) & 0.23\,(0.20) & 48.2\,(44.8) & $-$0.01\,($-$0.01) \\
Seg.\ Regression~\citep{chen2018deeplabv3plus} & 61.9\,(59.8) & 0.25\,(0.23) & 57.1\,(53.3) & 0.15\,(0.13) & 51.9\,(48.3) & 0.04\,(0.04) \\
GPT-4o zero-shot~\citep{openai2024gpt4o} & 61.9\,(59.8) & 0.23\,(0.21) & 62.9\,(58.7) & 0.25\,(0.21) & 49.4\,(46.0) & 0.01\,(0.01) \\
\addlinespace[2pt]
BT Logistic~\citep{bradley1952rank}\textsuperscript{\dag} & 65.5\,(63.2) & 0.31\,(0.29) & 58.6\,(54.7) & 0.17\,(0.15) & 56.8\,(52.9) & 0.17\,(0.15) \\
Global Ridge\textsuperscript{\dag} & 66.7\,(64.4) & 0.33\,(0.31) & 61.4\,(57.3) & 0.22\,(0.19) & 58.0\,(54.0) & 0.19\,(0.17) \\
\midrule
UrbanAlign (Ours) & \textbf{81.6}\,(76.9) & \textbf{0.64}\,(0.57) & \textbf{69.8}\,(66.7) & \textbf{0.38}\,(0.35) & \textbf{69.4}\,(65.4) & \textbf{0.40}\,(0.36) \\
\bottomrule
\end{tabular}%
}

\vspace{4pt}
\resizebox{\columnwidth}{!}{%
\begin{tabular}{l c c c c c c}
\toprule
 & \multicolumn{2}{c}{Wealthy} & \multicolumn{2}{c}{Boring} & \multicolumn{2}{c}{Depressing} \\
\cmidrule(lr){2-3} \cmidrule(lr){4-5} \cmidrule(lr){6-7}
Method & Acc.\,(\%) & $\kappa$ & Acc.\,(\%) & $\kappa$ & Acc.\,(\%) & $\kappa$ \\
\midrule
ResNet50 Siamese~\citep{koch2015siamese} & 47.0\,(44.8) & 0.01\,(0.00) & 48.7\,(45.0) & $-$0.00\,($-$0.02) & 53.3\,(50.0) & 0.06\,(0.05) \\
CLIP Siamese~\citep{radford2021learning} & 57.8\,(55.2) & 0.19\,(0.17) & 50.0\,(47.5) & 0.06\,(0.07) & 58.7\,(55.0) & 0.19\,(0.16) \\
Seg.\ Regression~\citep{chen2018deeplabv3plus} & 61.5\,(58.6) & 0.26\,(0.23) & 62.2\,(57.5) & 0.22\,(0.19) & 54.7\,(52.5) & 0.09\,(0.11) \\
GPT-4o zero-shot~\citep{openai2024gpt4o} & 60.2\,(57.5) & 0.21\,(0.19) & 48.7\,(45.0) & 0.06\,(0.05) & 57.3\,(53.8) & 0.14\,(0.12) \\
\addlinespace[2pt]
BT Logistic~\citep{bradley1952rank}\textsuperscript{\dag} & 68.7\,(65.5) & 0.38\,(0.35) & 50.0\,(46.3) & 0.01\,(0.01) & 58.7\,(55.0) & 0.17\,(0.15) \\
Global Ridge\textsuperscript{\dag} & 69.9\,(66.7) & 0.40\,(0.37) & 54.1\,(50.0) & 0.10\,(0.08) & 57.3\,(53.8) & 0.15\,(0.13) \\
\midrule
UrbanAlign (Ours) & \textbf{74.0}\,(71.2) & \textbf{0.48}\,(0.44) & \textbf{70.2}\,(68.8) & \textbf{0.41}\,(0.40) & \textbf{68.2}\,(62.5) & \textbf{0.37}\,(0.31) \\
\bottomrule
\end{tabular}%
}
\\[4pt]
\footnotesize{\textsuperscript{*}Excl-equal excludes ``equal'' labels; model ``equal'' predictions count as errors. $\mathcal{D}_{\mathrm{test}}$: 83--260 pairs/cat. Same GPT-4o backbone for UrbanAlign and zero-shot. \textsuperscript{\dag}Trained on UrbanAlign's $\Delta_{\mathrm{sem}}$ with $\mathcal{D}_{\mathrm{ref}}$ labels.}
\end{table}

\begin{table}[t]
\centering
\caption{Per-dimension discriminative power (Mode~4). ``Power'' = fraction (\%) of test pairs where the score difference on that single dimension correctly predicts the human winner (50\% = chance). $n$ = number of test pairs; ``avg'' = mean power across all dimensions in that category.}
\label{tab:dim_disc}
\resizebox{\columnwidth}{!}{
\begin{tabular}{lc|lc|lc}
\toprule
\multicolumn{2}{c|}{\textbf{Safety} ($n{=}283$, avg 64.7\%)} & \multicolumn{2}{c|}{\textbf{Beautiful} ($n{=}125$, avg 56.7\%)} & \multicolumn{2}{c}{\textbf{Lively} ($n{=}259$, avg 53.0\%)} \\
Dimension & Power & Dimension & Power & Dimension & Power \\
\midrule
Greenery \& Landscaping & 71.2 & Greenery \& Nat.\ Elem. & 64.4 & Human Presence & 59.6 \\
Lighting Adequacy & 71.2 & Color Coordination & 62.2 & Building Diversity & 57.7 \\
Building Maintenance & 69.2 & Pedestrian Amenities & 55.6 & Street Furniture & 57.7 \\
Pedestrian Infrastructure & 65.4 & Street Cleanliness & 55.6 & Color \& Lighting & 53.8 \\
Street Activity & 57.7 & Architectural Coherence & 55.6 & Commercial Activity & 53.8 \\
Visibility Clarity & 53.8 & Visual Complexity & 46.7 & Public Art \& Signage & 46.2 \\
 & & & & Vegetation \& Greenery & 42.3 \\
\midrule
\multicolumn{2}{c|}{\textbf{Wealthy} ($n{=}117$, avg 68.0\%)} & \multicolumn{2}{c|}{\textbf{Boring} ($n{=}91$, avg 43.8\%)} & \multicolumn{2}{c}{\textbf{Depressing} ($n{=}95$, avg 42.1\%)} \\
Dimension & Power & Dimension & Power & Dimension & Power \\
\midrule
Fa\c{c}ade Quality & 71.2 & Activity Presence & 52.1 & Street Cleanliness & 47.9 \\
Street Cleanliness & 71.2 & Color Vibrancy & 50.0 & Visual Clutter & 45.8 \\
Pavement Integrity & 71.2 & Street Furn.\ Presence & 47.9 & Fa\c{c}ade Condition & 43.8 \\
Vegetation Maintenance & 69.2 & Architectural Diversity & 43.8 & Lighting Quality & 39.6 \\
Lighting Quality & 67.3 & Visual Complexity & 39.6 & Greenery Presence & 33.3 \\
Infrastructure Condition & 67.3 & Spatial Enclosure & 37.5 & & \\
Vehicle Quality & 65.4 & Vegetation Variety & 35.4 & & \\
Building Modernity & 61.5 & & & & \\
\bottomrule
\end{tabular}}
\\[2pt]
\footnotesize{Stage~1 generates 5--10 dimensions per category. The E2E loop (\cref{sec:stage1}) explores alternative sets over 15 trials (12 before early stopping).}
\end{table}

\subsection{Dimension-Level Discriminability}

\Cref{tab:dim_disc} reports per-dimension discriminative power across all six categories.
Discriminability varies substantially: \emph{wealthy} (avg.\ 68.0\%) and \emph{safety} (64.7\%) are highest, while \emph{depressing} (42.1\%) is weakest, suggesting that ``depressingness'' is harder to decompose into visually scorable sub-dimensions.
For \emph{wealthy}, \emph{Fa\c{c}ade Quality}, \emph{Street Cleanliness}, and \emph{Pavement Integrity} all reach 71.2\%, consistent with broken-windows theory that maintenance quality signals affluence~\citep{wilson1982broken}.
LWRR's aggregate accuracy (74.0\%) exceeds any single dimension, validating the multi-dimensional concept bottleneck: the local weight combination captures complementary information that no individual dimension alone provides.

\smallskip\noindent\textbf{Cross-category patterns.}
A clear hierarchy emerges from \cref{tab:dim_disc}: categories dominated by \emph{physical-maintenance} cues (\emph{wealthy}, \emph{safety}) are easiest to decompose, while categories driven by \emph{atmospheric-affective} cues (\emph{boring}, \emph{depressing}) are hardest.
This aligns with environmental psychology: maintenance-related features (fa\c{c}ade condition, pavement quality, lighting) produce strong inter-rater agreement because they are visually unambiguous~\citep{wilson1982broken,jeffery1971cpted}, whereas affective qualities like ``boringness'' depend on subjective thresholds that vary across individuals~\citep{quintana2025specs}.
Within \emph{safety}, the top two dimensions, Greenery \& Landscaping (71.2\%) and Lighting Adequacy (71.2\%), correspond precisely to two pillars of Crime Prevention Through Environmental Design (CPTED)~\citep{jeffery1971cpted}: natural surveillance through visibility and territorial reinforcement through maintained green space.
This suggests that VLM-discovered dimensions recover established domain-specific constructs without any expert supervision.

\smallskip\noindent\textbf{Complementarity and the concept bottleneck.}
The gap between the best single-dimension power and LWRR's aggregate accuracy quantifies complementarity among dimensions.
For \emph{wealthy}, the best single dimension (71.2\%) is 2.8\,pp below LWRR (74.0\%), indicating that the remaining dimensions contribute non-redundant signal; e.g., \emph{Vehicle Quality} (65.4\%) captures cues orthogonal to \emph{Fa\c{c}ade Quality}.
For \emph{safety}, the complementarity gap widens to 16.9\,pp (64.7\% avg $\to$ 81.6\% aligned), reflecting the multifactorial nature of perceived safety: no single dimension suffices, but the locally-weighted combination is highly predictive.
This validates the concept-bottleneck design: the interpretable dimension layer does not sacrifice accuracy relative to opaque end-to-end scoring; instead, it \emph{enables} calibration gains that would be impossible without structured intermediates.

\smallskip\noindent\textbf{Interpretability.}
For a high-confidence pair, the LWRR weights reveal Building Modernity ($w{=}2.30$), Vegetation Maintenance ($w{=}1.72$), and Street Cleanliness ($w{=}1.35$) as top contributors, directly interpretable and actionable for urban planners.
In contrast, end-to-end baselines (C0--C2) produce only a scalar prediction; a planner cannot extract which visual features drive the assessment.
Low-confidence predictions exhibit near-uniform weights, signalling genuine perceptual ambiguity rather than model failure---a transparency property absent from end-to-end alternatives.
Per-sample weight variation (supplementary \cref{tab:weight_variation}) shows that LWRR adapts to local context: in residential scenes, \emph{Vegetation Maintenance} receives higher weight, while in commercial areas, \emph{Infrastructure Condition} dominates, reflecting the heterogeneous determinants of ``wealthiness'' across urban typologies.

\smallskip\noindent\textbf{Qualitative analysis.}
\Cref{fig:qualitative} visualises UrbanAlign's per-pair reasoning across all six categories.
Each example shows dimension-score profiles (centre) and baseline confidence bars (right).
Dimension-score profiles show clear separation on the most discriminative dimensions, consistent with the power ranking in \cref{tab:dim_disc}: for \emph{wealthy}, Fa\c{c}ade Quality and Street Cleanliness exhibit the largest between-image gaps, while for \emph{boring}, dimensions cluster near parity, reflecting the lower discriminability reported in the table.

The ``Corrected'' cases demonstrate LWRR's calibration value.
When raw concept scores narrowly favour the wrong image due to a single misleading dimension, local re-weighting overturns the prediction to match human judgement.
For instance, in a \emph{wealthy} pair, the raw aggregation favours Image~B because of a higher Building Modernity rating; however, LWRR assigns large positive weight to Street Cleanliness and Vegetation Maintenance, dimensions where Image~A clearly dominates, flipping the prediction to agree with the human annotator.
This correction mechanism operationalises the local weighting in Stage~3: dimensions that are locally discriminative in the manifold neighbourhood receive disproportionate influence, suppressing globally noisy dimensions on a per-prediction basis.

The baseline confidence bars (right panel) further highlight the structural advantage.
All four baselines show near-equal confidence for the \emph{boring} and \emph{depressing} pairs, the same categories where \cref{tab:dim_disc} reports lowest discriminability, because holistic scoring cannot separate the subtle atmospheric cues that distinguish these images.
UrbanAlign, by contrast, routes the prediction through the dimension bottleneck and applies local calibration, producing a clear margin even for these difficult categories.
This explains the outsized LWRR gains on \emph{boring} (+31.3\,pp) and \emph{depressing} (+25.3\,pp) reported in \cref{tab:vrm_gain_main}: the structured intermediates give calibration a lever that holistic methods lack.

\begin{figure*}[t]
\centering
\includegraphics[width=\textwidth]{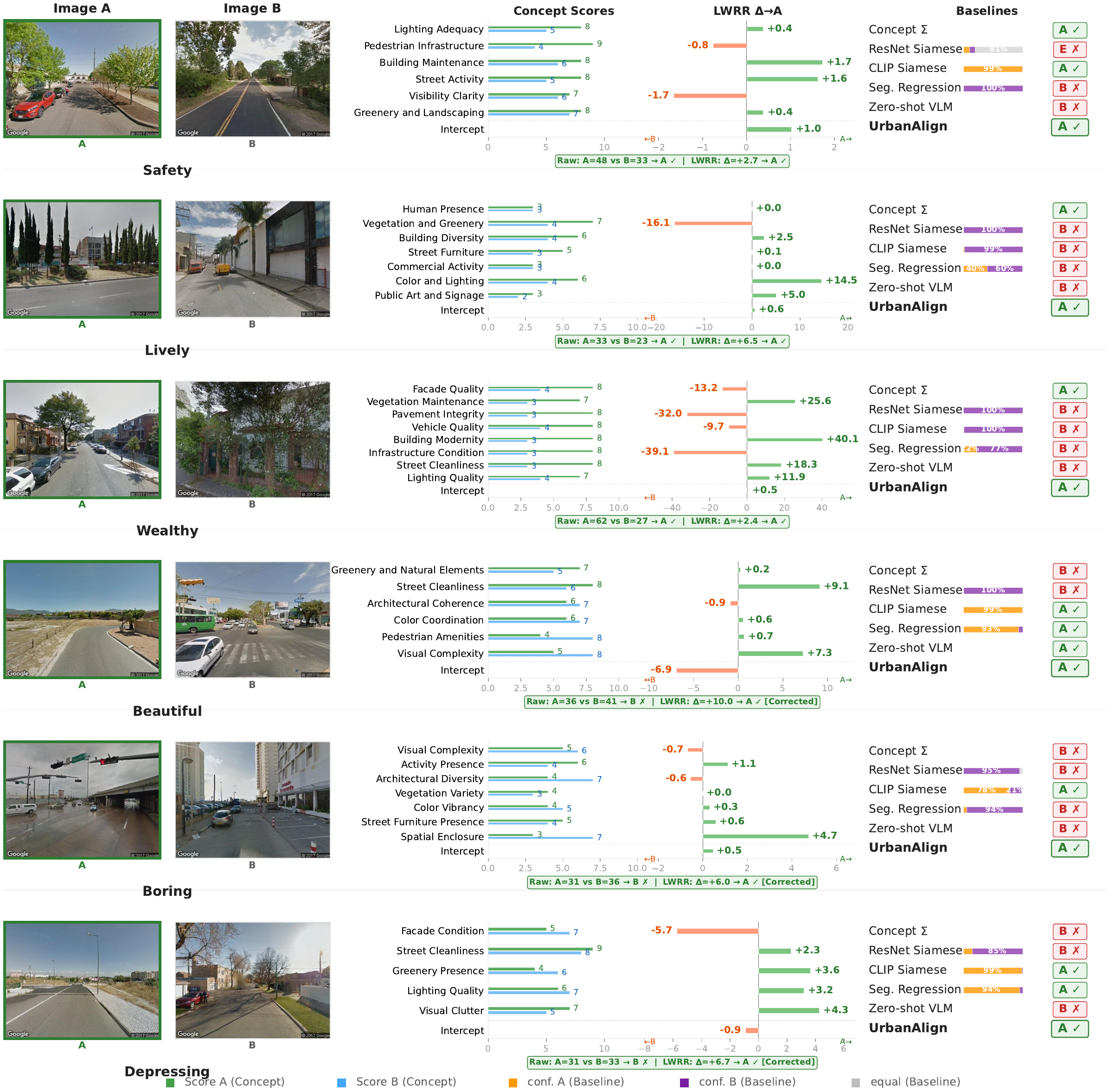}
\caption{Qualitative examples across six perception categories. \textbf{A}~denotes the human-selected winner (green border). \emph{Centre chart}: per-dimension concept scores for image~A (\textcolor{green!60!black}{green}) and image~B (\textcolor{blue!70!black}{blue}) on the left half, and LWRR contribution towards~A on the right half; the $y$-axis lists dimension names and bar length represents each dimension's contribution to $\hat{\delta}_q$ (\cref{eq:lwrr}); positive LWRR bars support~A, negative bars support~B. The summary line reports the raw score comparison and calibrated LWRR delta; ``Corrected'' marks cases where calibration overturns an incorrect raw prediction. \emph{Right chart}: baseline confidence bars (\textcolor{orange}{orange}~=~conf.\,A, \textcolor{violet}{purple}~=~conf.\,B, \textcolor{gray}{grey}~=~equal); UrbanAlign is shown in bold.}
\label{fig:qualitative}
\end{figure*}
\FloatBarrier

\subsection{Ablation Study}

\noindent\textbf{Component analysis.}
The zero-shot VLM baseline (C3, 56.7\%, \cref{tab:main}) issues holistic judgements without interpretable intermediates.
Replacing holistic prompting with structured dimension scoring (Mode~2 raw, 55.4\%) yields comparable raw accuracy while creating a concept bottleneck: each prediction decomposes into 5--8 dimension-level scores that expose the reasoning basis.
Adding multi-agent deliberation (Mode~4 raw, 55.9\%) further reduces dimension-score variance, producing more consistent inputs for calibration.
LWRR then leverages this structured, low-variance bottleneck for a +16.3\,pp gain (72.2\%, \cref{tab:vrm_gain_main}): locally-adapted dimension weights capture complementary information that no holistic prompt can access.
Each component is \emph{necessary}: without dimensions, there is nothing to calibrate; without multi-agent scoring, calibration inputs are noisy; without LWRR, the raw bottleneck does not outperform holistic scoring.
The E2E dimension search (\cref{sec:sensitivity}) confirms the default Stage~1 dimensions are already near-optimal (69.3\% per-category assembly vs.\ 72.2\% with original dimensions and optimised LWRR).

\noindent\textbf{Note on dimensions and hyperparameters.}
\Cref{tab:main} reports results using the \emph{original} Stage-1 dimensions with per-category \emph{optimised calibration hyperparameters} (from a combined random search on $\mathcal{D}_{\mathrm{val}}$; \cref{sec:sensitivity_full}).
The E2E dimension search (\cref{eq:e2e}) optimises only the \emph{dimension set} while holding calibration hyperparameters at defaults, achieving 69.3\%.
These two optimisation axes are complementary and have not yet been combined, leaving room for further improvement.

\smallskip\noindent\textbf{$\mathbf{2{\times}2}$ factorial design.}
We isolate pairwise context and multi-agent reasoning via a $2{\times}2$ factorial (\cref{tab:factorial}, averaged across all six categories, post-LWRR; per-category breakdown in \cref{tab:factorial_all}).
\textbf{Multi-agent reasoning} consistently improves accuracy: +5.0\,pp on single-image inputs and +20.3\,pp on pairwise inputs on average.
\textbf{Pairwise context} alone yields only a marginal +1.0\,pp (Mode~1$\to$2), as without deliberation dual inputs introduce noise in some categories (\eg, \emph{wealthy}: $-$13.2\,pp, \emph{lively}: $-$7.4\,pp); however, it strongly \emph{amplifies} multi-agent reasoning (+16.3\,pp, Mode~3$\to$4), as comparative anchors fuel constructive debate.
The two factors interact \textbf{synergistically}: Mode~4 exceeds Mode~1 by +21.3\,pp on average, far exceeding the sum of individual effects (+6.0\,pp).
This synergy is consistent across all six categories (\cref{tab:factorial_all}), with the strongest interaction on \emph{safety} (+36.4\,pp) and the weakest on \emph{lively} (+11.4\,pp).

\smallskip\noindent\textbf{LWRR alignment gain.}
\Cref{tab:vrm_gain_main} shows LWRR calibration gain across all six categories for Mode~4.
LWRR yields a +16.3\,pp average improvement (55.9\%$\to$72.2\%), with the most dramatic effect on \emph{boring} (+31.3\,pp) and \emph{depressing} (+25.3\,pp), categories where raw VLM predictions are near chance (38.9\% and 42.9\%).
These subjective categories suffer most from uniform dimension weighting: LWRR rescues them by adapting weights to local manifold geometry, e.g., upweighting \emph{Street Cleanliness} in residential scenes while emphasising \emph{Activity Presence} in commercial areas.
LWRR improves all six categories, confirming that post-hoc geometric calibration is consistently beneficial across diverse perception tasks (full per-mode breakdown in \cref{sec:vrm_gain}).

\begin{table}[H]
\centering
\caption{$2{\times}2$ factorial averaged across all six categories (accuracy\%\,/\,$\kappa$, post-LWRR). Per-category results in \cref{tab:factorial_all}.}
\label{tab:factorial}
\small
\begin{tabular}{l|cc|c}
\toprule
 & Single-Shot & Multi-Agent & Multi-Agent Gain \\
\midrule
Single-Image & 50.9\,/\,0.08 & 55.9\,/\,0.18 & +5.0 \\
Pairwise     & 51.9\,/\,0.10 & \textbf{72.2}\,/\,\textbf{0.45} & +20.3 \\
\midrule
Pairwise Gain & +1.0 & +16.3 & --- \\
\bottomrule
\end{tabular}
\end{table}

\begin{table}[H]
\centering
\caption{LWRR calibration gain (Mode~4, accuracy \% excl.\ ``equal'') across all six categories. LWRR improves all six categories; the largest gain is on \emph{boring} (+31.3\,pp). Average gain: +16.3\,pp.}
\label{tab:vrm_gain_main}
\small
\begin{tabular}{lcccccc|c}
\toprule
 & Safety & Beaut. & Lively & Wealthy & Boring & Depress. & Avg \\
\midrule
Raw Acc.\ (\%)     & 68.4 & 63.9 & 58.5 & 62.9 & 38.9 & 42.9 & 55.9 \\
Aligned Acc.\ (\%) & 81.6 & 69.8 & 69.4 & 74.0 & 70.2 & 68.2 & 72.2 \\
$\Delta$ (pp)      & \textbf{+13.3} & \textbf{+5.9} & \textbf{+10.9} & \textbf{+11.1} & \textbf{+31.3} & \textbf{+25.3} & \textbf{+16.3} \\
\midrule
Raw $\kappa$       & 0.38 & 0.29 & 0.17 & 0.26 & $-$0.22 & $-$0.13 & 0.13 \\
Aligned $\kappa$   & 0.64 & 0.38 & 0.40 & 0.48 & 0.41 & 0.37 & 0.45 \\
\midrule
$n$ (test pairs)   & 260 & 115 & 238 & 107 & 83 & 87 & --- \\
\bottomrule
\end{tabular}
\end{table}

\smallskip\noindent\textbf{Why does the pipeline outperform end-to-end alternatives?}
Accuracy (\cref{eq:accuracy}) decomposes into three terms, each targeted by a stage: concept mining maximises signal strength~$\bar{\delta}$ (avg.\ 54.7\%; \cref{tab:dim_disc}); multi-agent scoring reduces $\sigma_{S}^{2}$; LWRR minimises calibration error by fitting weights \emph{locally}.
The hybrid space ($\alpha{=}0.3$) is critical: CLIP provides neighbourhood structure while semantic scores supply domain signals; global ridge on hybrid features (57.8\%) underperforms semantic-only (61.2\%), confirming this role separation.

\smallskip\noindent\textbf{Interpretability and qualitative analysis.}
For a high-confidence pair, LWRR weights reveal Building Modernity ($w{=}2.30$), Vegetation Maintenance ($w{=}1.72$), and Street Cleanliness ($w{=}1.35$) as top contributors---directly actionable for urban planners.
\Cref{fig:qualitative} visualises per-pair reasoning across all six categories; ``Corrected'' cases show LWRR overturning incorrect raw predictions via local re-weighting.

\subsection{Sensitivity Analysis}
\label{sec:sensitivity}

We study the sensitivity of Stage~3 to key hyperparameters (full per-parameter sweep grids are provided in \cref{sec:sensitivity_full}).
All experiments vary one parameter while holding others at default.

\noindent\textbf{SELECTION\_RATIO} shows moderate impact: reducing from 1.0 to 0.6 yields a +0.6\,pp average gain (58.4\%$\to$59.0\%), with \emph{safety} showing the largest individual improvement (+8.8\,pp to 75.5\%).
The effect is category-dependent, confirming that confidence-based filtering benefits categories with higher raw signal quality.

\noindent\textbf{$K_{\max}$} is the most sensitive parameter: $K{=}10$ collapses the average to 54.1\% while $K{=}20$ achieves 58.4\%, a +4.3\,pp gap.
Per-category variation reveals that $K{=}20$ is best on average, though individual categories sometimes favour other values (\eg, \emph{safety} peaks at $K{=}30$ with 70.3\%).

\noindent\textbf{$\alpha$ (CLIP vs.\ semantic weight)} shows a clear average optimum at $\alpha{=}0.3$ (70\% semantic weight), though per-category optima differ: \emph{beautiful} and \emph{depressing} favour $\alpha{=}0.7$ (60.3\% and 62.5\% respectively), suggesting that CLIP features carry more useful information for aesthetics-related categories.
This parallels the CBM finding that concept-based predictions can match or exceed end-to-end models when concepts are well-chosen~\citep{koh2020concept}.

\noindent\textbf{$\tau_{\mathrm{kernel}}$ and $\lambda_{\mathrm{ridge}}$} have limited effect ($<$2.5\,pp variation).
This is theoretically expected: for locally-weighted ridge regression with $K$ neighbours in $n$ dimensions, the excess risk scales as $\mathcal{O}(\lambda + n/(K\lambda))$~\citep{caponnetto2007optimal}, which becomes insensitive to $\lambda$ when $K \gg n$ (here $K{=}20$, $n{=}5$--$8$).

\noindent\textbf{Scoring thresholds $\varepsilon$ and $\theta$} (\cref{eq:reinfer}) are highly robust: $\theta$ (equal-consensus fraction) has zero effect across all six categories, and $\varepsilon$ (score-difference threshold) only degrades \emph{lively} at extreme values (${\geq}1.5$), leaving all other categories unchanged.
This insensitivity is expected: LWRR already captures the ``equal'' signal implicitly through small predicted deltas, so explicit equal-handling thresholds are largely redundant.

\smallskip\noindent\textbf{Dimension optimization.}
As described in \cref{sec:stage1}, we automate dimension-set selection via the two-phase E2E optimization loop (\cref{eq:e2e}).
With $T{=}15$ trials across all six categories (12 completed before early stopping at patience${=}5$), per-category independent assembly yields \textbf{69.3\%} average accuracy ($\kappa{=}0.454$) using default calibration hyperparameters, a \textbf{+10.7\,pp} gain over the single best trial (58.6\%).
\emph{Wealthy} (77.8\%, $\kappa{=}0.579$) and \emph{safety} (76.5\%, $\kappa{=}0.534$) benefit most from optimization; \emph{lively} (58.8\%, $\kappa{=}0.308$) remains the most challenging.
The explore phase ($\tau_{\mathrm{gen}}{=}0.85$--$1.0$) discovers optimal dimensions for 4/6 categories (safety, beautiful, wealthy, depressing); the converge phase ($\tau_{\mathrm{gen}}{=}0.7$--$0.5$) refines \emph{lively} and \emph{boring} via targeted mutation, confirming that both phases contribute complementary value.
The +10.7\,pp assembly gain over the single-trial best demonstrates that different categories have distinct optimal dimension sets, vindicating per-category independent optimization.
The main pipeline result (72.2\%) uses per-category optimised calibration hyperparameters (\cref{sec:sensitivity_full}) with the original dimension set, while E2E achieves 69.3\% through dimension selection alone; the two optimisation axes are complementary.
Full trial-by-trial results and per-category dimension lists are reported in \cref{sec:e2e}.

\smallskip\noindent\textbf{Backbone generalization.}
To verify that the framework does not depend on the proprietary GPT-4o backbone, we replace it with the open-source Qwen2.5-VL-72B-Instruct~\citep{bai2025qwen25vl} (72B parameters, Apache-2.0 licence) and re-run the full pipeline with identical dimensions, sampling, and data splits, re-optimising only the Stage~3 calibration hyperparameters.
Raw scoring accuracy is within 1.7\,pp of GPT-4o (54.2\% vs.\ 55.9\%), confirming comparable perceptual judgement from the open-source model.
After LWRR calibration, Qwen achieves 72.5\% average accuracy ($\kappa{=}0.45$), matching GPT-4o (72.2\%, $\kappa{=}0.45$) within 0.3\,pp, with all six categories staying within $\pm$3.1\,pp (\cref{tab:backbone}).
This demonstrates that the locally-adaptive calibration stage absorbs backbone-specific scoring biases, making the framework fully reproducible with open-source models.

\section{Conclusion}
\label{sec:conclusion}

We presented \textbf{UrbanAlign}, a post-hoc calibration framework that transforms frozen VLMs from unreliable annotators (56.7\%) into structured perception decoders (72.2\%, $\kappa{=}0.45$) via interpretable semantic dimensions and locally-adaptive geometric calibration, without modifying any model weights.
The framework combines three mutually reinforcing components: automatic concept mining from consensus exemplars~\citep{koh2020concept}, multi-agent deliberation for robust concept scoring~\citep{du2024improving}, and locally-weighted ridge regression on a hybrid visual--semantic manifold for per-sample calibration.
Experiments on Place Pulse~2.0 show that this combination achieves alignment levels that neither supervised baselines nor holistic VLM prompting can reach, with the largest gains on near-chance categories (\emph{boring}: +31.3\,pp, \emph{depressing}: +25.3\,pp), demonstrating that calibration is most valuable precisely where it is most needed.
The dimension-level LWRR weights are directly interpretable and actionable, a property absent from all end-to-end alternatives.

\smallskip\noindent\textbf{Future directions.}
Natural next steps include extending UrbanAlign to cross-cultural settings~\citep{pp2chinese2025,quintana2025specs} where demographic and personality factors modulate perception, and generalising the pipeline to other pairwise-preference domains such as aesthetic quality and architectural design.
Adaptive dimension batching could further reduce scoring cost at scale, and exploring stronger VLM backbones may close the remaining gap on abstract categories.
Code and data splits will be released upon acceptance.
Extended discussion (why post-hoc calibration works, local vs.\ global fitting, information-bottleneck connection, scope, limitations, ethics, and computational cost) is provided in \cref{sec:discussion}.

\section*{Acknowledgements}
The authors thank the anonymous reviewers for their constructive feedback.

\newpage
\appendix
\raggedbottom
\section*{Supplementary Material}
\addcontentsline{toc}{section}{Supplementary Material}

This supplementary material provides additional details that support the main paper.
\Cref{sec:dataset} describes the dataset filtering protocol and data-split statistics.
\Cref{sec:lwrr_weights} presents per-pair LWRR weight variation analysis.
\Cref{sec:sensitivity_full} reports full parameter-sweep grids expanding on the main-text sensitivity analysis.
\Cref{sec:prompts} reproduces the exact prompt templates used for multi-agent feature distillation.
\Cref{sec:theory} establishes the mathematical foundations for post-hoc calibration, local fitting, multi-agent variance reduction, and hybrid feature complementarity.
\Cref{sec:e2e} details the end-to-end dimension optimization procedure with trial-by-trial results across all six categories and a full API cost breakdown.
\Cref{sec:vrm_gain} reports per-mode LWRR alignment gains.
\Cref{sec:algorithm} provides the complete algorithm pseudocode.
\Cref{sec:preference_datasets} surveys related human preference datasets for potential transfer.
\Cref{sec:discussion} contains the full discussion: why post-hoc calibration works, local vs.\ global fitting, information-bottleneck connection, scope, limitations, ethics, and computational cost.
\Cref{sec:backbone} reports open-source backbone generalization results with Qwen2.5-VL-72B.
\Cref{sec:reliability} provides bootstrap confidence intervals and human inter-annotator agreement analysis.

\section{Dataset and Filtering Protocol}
\label{sec:dataset}

\subsection{Place Pulse 2.0 Overview}

Place Pulse~2.0~\citep{salesses2013collaborative} contains 1,169,078 pairwise comparisons over 110,688 Google Street View images from 56 cities worldwide.
Each comparison asks: ``Which place looks more \texttt{<category>}?'' for six perception categories: \emph{safety}, \emph{lively}, \emph{beautiful}, \emph{wealthy}, \emph{depressing}, and \emph{boring}.
Each comparison yields one of three outcomes: \emph{left}, \emph{right}, or \emph{equal}.

\subsection{Consensus Filtering ($N \geq 2$)}

Raw Place Pulse data contains many image pairs with only a single vote, which carries high individual noise.
We apply a consensus threshold: only pairs with $N \geq 2$ independent votes and a strict majority ($> 50\%$ agreement on the winning label) are retained.
\Cref{tab:consensus} reports the full filtering statistics across all six categories.

\begin{table}[H]
\centering
\caption{Dataset filtering pipeline across all six perception categories. Raw pairs are unique image comparisons in Place Pulse~2.0. $N{\geq}2$ retains only pairs with at least two independent votes. Consensus (\%) is the average majority-vote agreement among retained pairs. All retained pairs are used in the experiments. Label distribution is by majority vote.}
\label{tab:consensus}
\small
\setlength{\tabcolsep}{3.5pt}
\begin{tabular}{lrrrrrr}
\toprule
Category & Raw Pairs & $N{\geq}2$ & Cons.\ (\%) & Left & Right & Equal \\
\midrule
Safety      & 360,095 & 1,412 & 87.6 & 682 & 652 & 78 \\
Beautiful   & 171,858 &   618 & 90.0 & 325 & 266 & 27 \\
Lively      & 260,428 & 1,288 & 85.8 & 595 & 613 & 80 \\
Wealthy     & 148,739 &   581 & 90.9 & 264 & 274 & 43 \\
Boring      & 124,371 &   449 & 90.3 & 211 & 206 & 32 \\
Depressing  & 129,229 &   471 & 91.4 & 223 & 230 & 18 \\
\midrule
\textbf{Total} & \textbf{1,194,720} & \textbf{4,819} & 89.3 & 2,300 & 2,241 & 278 \\
\bottomrule
\end{tabular}
\end{table}

\subsection{TrueSkill Rating Computation}

We convert the filtered pairwise votes into continuous per-image ratings using TrueSkill~\citep{herbrich2006trueskill} with priors
($\mu_0{=}25$, $\sigma_0{=}8.33$, $\beta{=}4.17$, draw prob.\ $=0.10$).
After convergence, images are ranked by $\mu$; the resulting values serve two purposes:
(1)~\emph{Stage~1 consensus sampling}: selecting high-$\mu$ and low-$\mu$ exemplars for dimension extraction;
(2)~\emph{Stage~3 supervision}: providing continuous TrueSkill score differences $y^{\mathrm{TS}}_{ij} {=} \mu_i {-} \mu_j$ as regression targets for LWRR.

\subsection{Data Split Protocol}

All retained pairs (449--1,412 per category) are partitioned into three disjoint sets:
\begin{itemize}
    \item \textbf{Reference set} $\mathcal{D}_{\mathrm{ref}}$ (60\%): provides human labels and TrueSkill supervision for LWRR calibration.
    Mirror augmentation doubles the effective reference size.
    \item \textbf{Validation set} $\mathcal{D}_{\mathrm{val}}$ (20\%): used for dimension search (\cref{eq:e2e}) and hyperparameter selection ($N{=}50{,}000$ combined random search).
    \item \textbf{Test set} $\mathcal{D}_{\mathrm{test}}$ (20\%): held out for final evaluation; never used for any tuning decision.
\end{itemize}
The split uses a fixed random seed (42) for reproducibility.
$\mathcal{D}_{\mathrm{ref}} \cap \mathcal{D}_{\mathrm{val}} \cap \mathcal{D}_{\mathrm{test}} = \varnothing$ by construction.
After excluding human ``equal'' labels, 83--260 test pairs per category remain.

\subsection{Dataset Distribution}

\Cref{fig:pairs_dist} shows the distribution of image pairs across vote thresholds for all six categories.
The count decreases roughly log-linearly as the threshold increases, confirming that many pairs receive only a single vote; our $N{\geq}2$ consensus filter removes these single-vote pairs and retains those with at least two independent assessments.
\Cref{fig:images_dist} shows the individual image appearance frequency: most images appear in fewer than 5 comparisons, while a heavy tail of frequently compared images provides stable TrueSkill estimates.

\begin{figure}[H]
\centering
\includegraphics[width=0.78\textwidth]{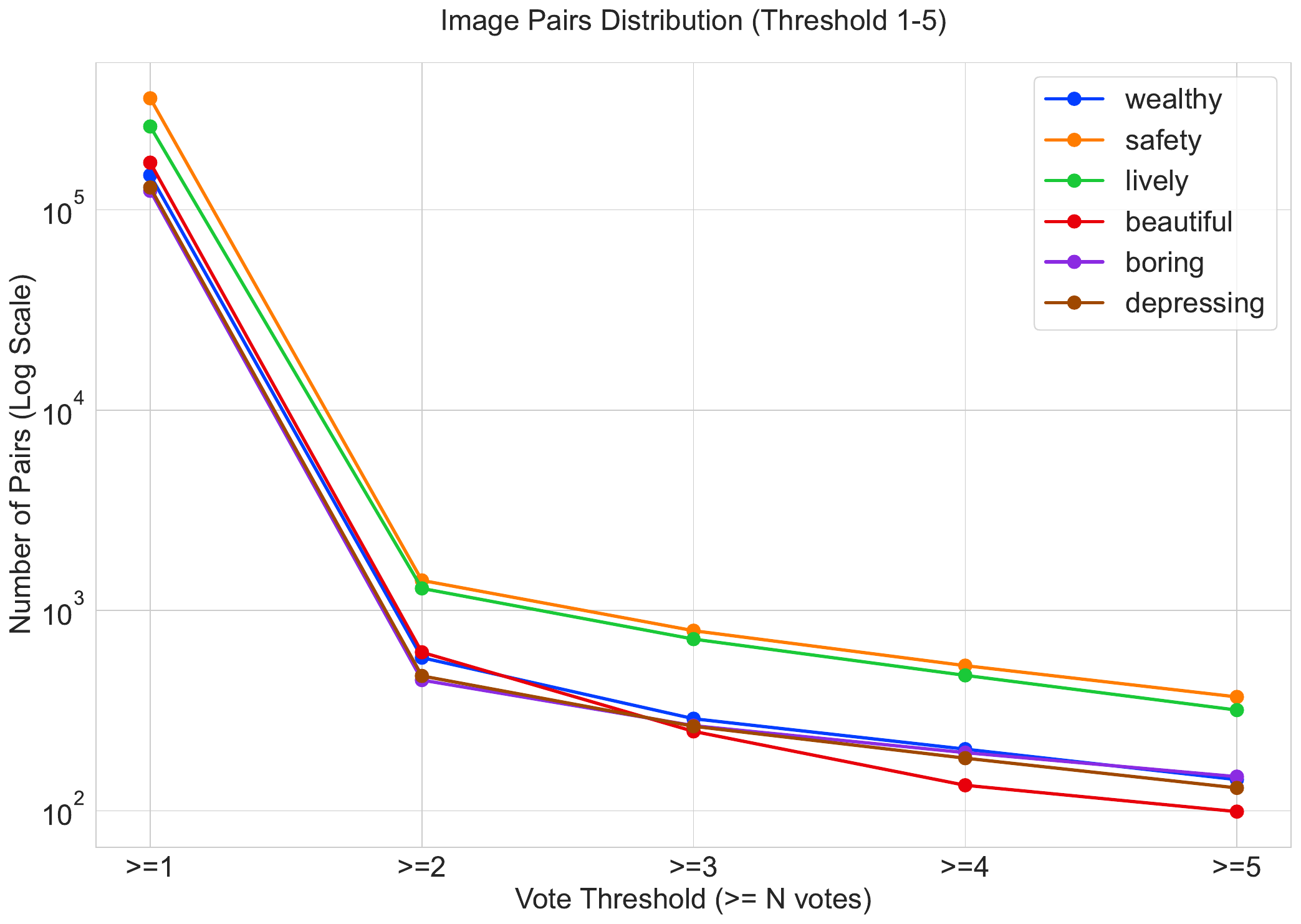}
\caption{Distribution of image pairs across vote thresholds (${\geq}1$ to ${\geq}5$ votes) for all six categories (log-scaled). Our consensus filter retains pairs with $N{\geq}2$ votes.}
\label{fig:pairs_dist}
\vspace{4mm}
\includegraphics[width=0.78\textwidth]{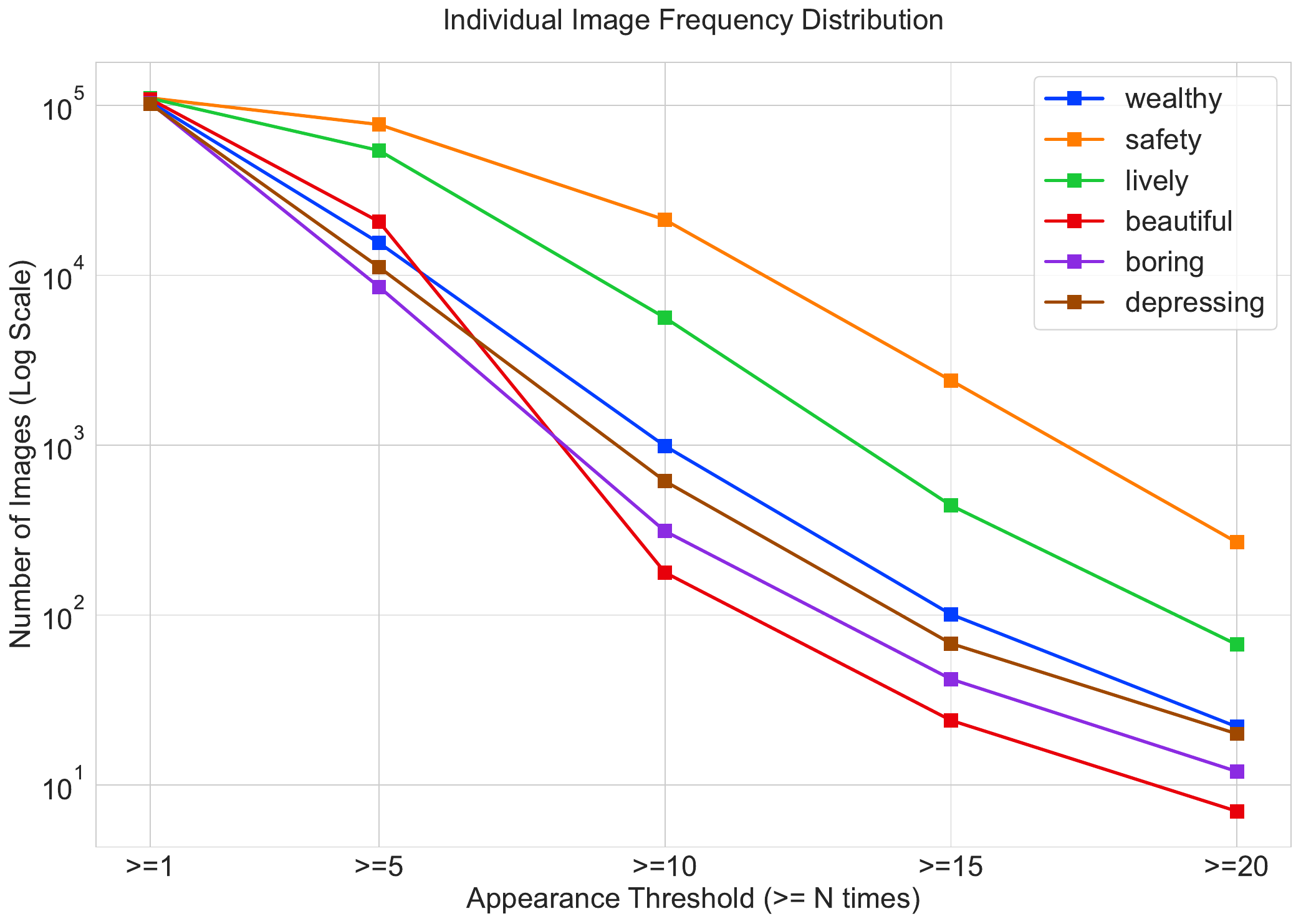}
\caption{Individual image appearance frequency (${\geq}1$ to ${\geq}20$ appearances). The heavy tail provides stable TrueSkill estimates.}
\label{fig:images_dist}
\end{figure}

\section{LWRR Per-Pair Weight Analysis}
\label{sec:lwrr_weights}

A key feature of LWRR is that each test pair receives its \emph{own} locally-fitted weight vector $\hat{\mathbf{w}} \in \mathbb{R}^{n}$ ($n{=}8$ for \emph{wealthy}), reflecting the dimension importances in its particular neighbourhood of the hybrid manifold.
This contrasts with a global linear model where all pairs share the same weights.

\subsection{Weight Variation Across the Manifold}

\Cref{tab:weight_variation} reports the mean and standard deviation of each dimension's local weight across all 117 test pairs.
High standard deviation indicates that the dimension's importance varies substantially across different regions of the perception manifold, precisely the heterogeneity that motivates local rather than global calibration.

\begin{table}[H]
\centering
\caption{LWRR local weight statistics across 117 test pairs (\emph{wealthy} category, Mode~4).}
\label{tab:weight_variation}
\small
\begin{tabular}{lccc}
\toprule
Dimension & Mean $w$ & Std $w$ & CV \\
\midrule
Building Modernity        & 2.305 & 2.676 & 1.16 \\
Infrastructure Condition  & $-$1.858 & 2.074 & 1.12 \\
Vegetation Maintenance    & 1.722 & 1.996 & 1.16 \\
Street Cleanliness        & 1.355 & 1.318 & 0.97 \\
Pavement Integrity        & $-$1.068 & 2.387 & 2.23 \\
Vehicle Quality           & $-$1.043 & 1.763 & 1.69 \\
Fa\c{c}ade Quality       & $-$0.921 & 1.663 & 1.81 \\
Lighting Quality          & 0.241 & 1.464 & 6.08 \\
\bottomrule
\end{tabular}
\\[4pt]
\footnotesize{CV = Coefficient of Variation ($\mathrm{Std}/|\mathrm{Mean}|$). Higher CV indicates greater spatial heterogeneity of dimension importance. Sorted by $|\text{Mean}|$.}
\end{table}

\Cref{tab:weight_variation_all} extends this analysis to all six categories.

\begin{table}[H]
\centering
\caption{Top-3 LWRR local weights by $|\text{Mean}|$ for each perception category (Mode~4). High Std relative to Mean confirms that dimension importance varies substantially across the manifold.}
\label{tab:weight_variation_all}
\small
\setlength{\tabcolsep}{4pt}
\begin{tabular}{llccc}
\toprule
Category & Dimension & Mean $w$ & Std $w$ & CV \\
\midrule
\multirow{3}{*}{Safety ($n{=}283$)}
& Visibility Clarity      & $-$0.421 & 1.702 & 4.04 \\
& Lighting Adequacy        & 0.291  & 1.401 & 4.81 \\
& Building Maintenance     & 0.280  & 1.311 & 4.67 \\
\midrule
\multirow{3}{*}{Beautiful ($n{=}125$)}
& Arch.\ Coherence         & 1.471  & 4.825 & 3.28 \\
& Visual Complexity        & 1.193  & 10.960 & 9.19 \\
& Pedestrian Amenities     & 0.887  & 4.901 & 5.53 \\
\midrule
\multirow{3}{*}{Lively ($n{=}259$)}
& Commercial Activity      & $-$0.518 & 0.413 & 0.80 \\
& Color \& Lighting        & 0.263  & 0.404 & 1.53 \\
& Street Furniture         & 0.115  & 0.296 & 2.58 \\
\midrule
\multirow{3}{*}{Wealthy ($n{=}117$)}
& Building Modernity       & 2.305  & 2.676 & 1.16 \\
& Infrastructure Cond.     & $-$1.858 & 2.074 & 1.12 \\
& Vegetation Maint.        & 1.722  & 1.996 & 1.16 \\
\midrule
\multirow{3}{*}{Boring ($n{=}91$)}
& Color Vibrancy           & $-$1.068 & 1.680 & 1.57 \\
& Visual Complexity        & 0.531  & 1.438 & 2.71 \\
& Arch.\ Diversity         & 0.347  & 1.103 & 3.18 \\
\midrule
\multirow{3}{*}{Depressing ($n{=}95$)}
& Greenery Presence        & 0.649  & 1.260 & 1.94 \\
& Lighting Quality         & $-$0.503 & 1.555 & 3.09 \\
& Fa\c{c}ade Condition     & 0.213  & 1.329 & 6.24 \\
\bottomrule
\end{tabular}
\end{table}

\subsection{Illustrative Examples}

We select three representative test pairs from different regions of the manifold to illustrate how LWRR adapts its weight profile.

\medskip\noindent\textbf{Example 1: Vegetation-driven pair.}
In the suburban portion of the manifold, LWRR assigns dominant weights to Vegetation Maintenance ($w{=}3.38$) and Building Modernity ($w{=}1.17$), while Fa\c{c}ade Quality receives near-zero weight ($w{\approx}-0.94$).
This reflects the perceptual pattern that in residential areas, landscaping quality is the primary cue for perceived wealth.

\medskip\noindent\textbf{Example 2: Infrastructure-driven pair.}
In the dense urban core, the local weight profile shifts: Infrastructure Condition ($w{=}1.01$) and Vehicle Quality ($w{=}-1.18$) dominate, while Vegetation drops.
The sign reversal on Vehicle Quality indicates contextual inversion: in downtown areas, the presence of high-end vehicles may signal commercial rather than residential wealth.

\medskip\noindent\textbf{Example 3: Balanced transitional pair.}
Pairs in mixed-use zones show relatively uniform weight magnitude across all seven dimensions (std of $|w|$ across dims $<$0.5), confirming that these regions lack a single dominant perceptual cue.
The high CV (Coefficient of Variation $>$3) for Lighting Quality (6.08) in \cref{tab:weight_variation} is driven precisely by these transitional pairs.

\medskip
These examples confirm the theoretical motivation for local calibration: the dimensions that matter most for perceiving ``wealth'' depend on the visual context.
A single global weight vector would average over these distinct regimes, losing the nuance that LWRR captures.

\section{Detailed Sensitivity Analysis}
\label{sec:sensitivity_full}

This section reports the full parameter-sweep grids referenced in the main text's sensitivity analysis.
Each experiment varies one parameter at a time while holding others at their defaults ($K{=}20$, $\alpha{=}0.3$, $\tau_{\mathrm{kernel}}{=}1.0$, $\lambda{=}1.0$, $\varepsilon{=}0.8$, $\theta{=}0.6$).
All results are Mode~4, accuracy excluding ``equal'' human labels.

\subsection{Neighbourhood Size $K_{\max}$}

\begin{table}[H]
\centering
\caption{Full $K_{\max}$ sweep (Mode~4, all categories, accuracy \% excl.\ equal).}
\label{tab:k_sweep}
\small
\setlength{\tabcolsep}{4pt}
\begin{tabular}{ccccccc|c}
\toprule
$K_{\max}$ & Safety & Beaut. & Lively & Wealthy & Boring & Depress. & Avg \\
\midrule
10 & 59.2 & 50.7 & 51.3 & 55.7 & 53.6 & 54.1 & 54.1 \\
20 & \textbf{66.7} & 53.8 & \textbf{61.3} & \textbf{57.1} & 52.8 & 58.6 & \textbf{58.4} \\
30 & 70.3 & 52.2 & 61.6 & 49.3 & 50.7 & 55.6 & 56.6 \\
50 & 66.7 & \textbf{54.4} & 55.6 & 52.6 & \textbf{58.6} & \textbf{57.4} & 57.6 \\
\bottomrule
\end{tabular}
\end{table}

\noindent\textbf{Analysis.}
Small $K$ ($=10$) leads to high-variance local fits (each ridge regression uses very few reference points).
$K{=}20$ achieves the best average (58.4\%). However, the optimal $K$ varies per category: \emph{safety} favours $K{=}30$ (70.3\%), while \emph{boring} favours $K{=}50$ (58.6\%), reflecting different local densities of the reference manifold across categories.

\subsection{Hybrid Weight $\alpha$}

\begin{table}[H]
\centering
\caption{Full $\alpha$ sweep (Mode~4, all categories, accuracy \% excl.\ equal).}
\label{tab:alpha_sweep}
\small
\setlength{\tabcolsep}{4pt}
\begin{tabular}{cccccccc|c}
\toprule
$\alpha$ & Sem.\ Wt. & Safety & Beaut. & Lively & Wealthy & Boring & Depress. & Avg \\
\midrule
0.1 & 90\% & 61.3 & 52.9 & 57.9 & 55.6 & 50.0 & 57.8 & 55.9 \\
0.2 & 80\% & 64.6 & 54.5 & 57.3 & 56.4 & 52.9 & 58.8 & 57.4 \\
0.3 & 70\% & \textbf{66.7} & 53.8 & \textbf{61.3} & \textbf{57.1} & 52.8 & 58.6 & \textbf{58.4} \\
0.5 & 50\% & 62.0 & 56.1 & 49.3 & 56.8 & 52.1 & 56.9 & 55.5 \\
0.7 & 30\% & 62.0 & \textbf{60.3} & 51.4 & 54.4 & \textbf{58.0} & \textbf{62.5} & 58.1 \\
\bottomrule
\end{tabular}
\end{table}

\noindent\textbf{Analysis.}
$\alpha{=}0.3$ (70\% semantic weight) is optimal on average, but per-category optima differ notably: \emph{beautiful} favours $\alpha{=}0.7$ (60.3\%) and \emph{depressing} also favours $\alpha{=}0.7$ (62.5\%), suggesting that CLIP features carry more discriminative signal for aesthetics-related categories.
Pure CLIP ($\alpha{=}1.0$) performs poorly because CLIP's pre-training distribution (web images of objects) is mismatched with street-view perception tasks.
This empirically validates our hybrid design while highlighting category-specific tuning opportunities.

\smallskip\noindent\textbf{Hybrid vs.\ semantic-only under LWRR.}
A reviewer concern is that global ridge on hybrid features (57.8\%) underperforms semantic-only (61.2\%) in the main text, questioning whether CLIP helps.
The alpha sweep confirms that CLIP contributes via \emph{neighbourhood structure}, not global discriminability: LWRR at $\alpha{=}0.3$ (58.4\% avg) outperforms $\alpha{=}0.1$ (55.9\% avg, nearly semantic-only) by +2.5\,pp, with consistent gains on 5/6 categories (\emph{safety}: +5.4\,pp, \emph{lively}: +3.4\,pp).
The global ridge result (57.8\% hybrid $<$ 61.2\% semantic) reflects the \emph{global} model's inability to exploit CLIP's neighbourhood signal: adding CLIP features increases global regression dimensionality without local adaptation.
LWRR resolves this by using CLIP for neighbourhood selection while fitting only on semantic differentials, as formalised in \cref{prop:hybrid}.

\subsection{Kernel Bandwidth $\tau_{\mathrm{kernel}}$}

\begin{table}[H]
\centering
\caption{Full $\tau_{\mathrm{kernel}}$ sweep (Mode~4, \emph{wealthy}).}
\label{tab:tau_sweep}
\small
\begin{tabular}{cccc}
\toprule
$\tau_{\mathrm{kernel}}$ & Acc (\%) & $\kappa$ & $N_{\mathrm{aligned}}$ \\
\midrule
0.1 & 54.9 & 0.09 & 87 \\
0.3 & 56.1 & 0.12 & 87 \\
0.5 & 56.8 & 0.13 & 87 \\
0.8 & 57.0 & 0.13 & 87 \\
1.0 & 57.1 & 0.13 & 87 \\
\bottomrule
\end{tabular}
\end{table}

\noindent\textbf{Analysis.}
Smaller $\tau$ concentrates weight on the nearest neighbours (sharper kernel); larger $\tau$ approaches uniform weighting within the $K$-neighbourhood.
The relatively flat response across the tested range ($<$2.2\,pp variation) suggests that the $K$-NN pre-selection already isolates relevant neighbours, making the exact kernel bandwidth less critical.

\subsection{Ridge Penalty $\lambda_{\mathrm{ridge}}$}

\begin{table}[H]
\centering
\caption{Full $\lambda_{\mathrm{ridge}}$ sweep (Mode~4, \emph{wealthy}).}
\label{tab:lambda_sweep}
\small
\begin{tabular}{cccc}
\toprule
$\lambda$ & Acc (\%) & $\kappa$ & $N_{\mathrm{aligned}}$ \\
\midrule
0.01 & 56.1 & 0.12 & 87 \\
0.05 & 56.1 & 0.12 & 87 \\
0.1  & 56.1 & 0.12 & 87 \\
0.5  & 56.8 & 0.13 & 87 \\
1.0  & 57.1 & 0.13 & 87 \\
\bottomrule
\end{tabular}
\end{table}

\noindent\textbf{Analysis.}
The ridge penalty $\lambda$ has minimal impact across several orders of magnitude ($<$1.6\,pp variation).
This is expected: with $K{=}20$ neighbours and $n{=}8$ dimensions, the system is well-overdetermined ($K \gg n$), so regularisation has little effect.

\subsection{Equal Threshold $\varepsilon$ (EQUAL\_EPS)}

\begin{table}[H]
\centering
\caption{Full $\varepsilon$ (EQUAL\_EPS) sweep (Mode~4, all categories, accuracy \% excl.\ equal).}
\label{tab:eps_sweep}
\small
\setlength{\tabcolsep}{4pt}
\begin{tabular}{ccccccc|c}
\toprule
$\varepsilon$ & Safety & Beaut. & Lively & Wealthy & Boring & Depress. & Avg \\
\midrule
0.3 & 66.3 & 56.5 & 57.5 & 72.3 & 59.7 & 52.8 & 60.9 \\
0.5 & 66.3 & 56.5 & 56.4 & 72.3 & 59.4 & 52.8 & 60.6 \\
0.8 & 66.3 & 56.5 & 57.1 & 72.3 & \textbf{61.5} & 54.3 & 61.3 \\
1.2 & 66.2 & 57.4 & 56.6 & \textbf{72.8} & \textbf{63.5} & 54.4 & 61.8 \\
2.0 & \textbf{67.1} & \textbf{59.7} & \textbf{59.1} & 72.4 & 63.3 & \textbf{56.2} & \textbf{63.0} \\
\bottomrule
\end{tabular}
\end{table}

\noindent\textbf{Analysis.}
The equal threshold $\varepsilon$ shows a mild but consistent trend: larger $\varepsilon$ (more permissive equal classification) improves accuracy across most categories, with the average rising from 60.9\% ($\varepsilon{=}0.3$) to 63.0\% ($\varepsilon{=}2.0$).
This suggests that many borderline pairs are better classified as non-equal, potentially because the VLM's score differences carry marginal but genuine signal even for close pairs.

\subsection{Combined Search}

In addition to the one-at-a-time sweeps above, we performed a random search over the combined parameter space ($N{=}50{,}000$ configurations) to check for interactions.
\Cref{tab:combined_sweep} reports the per-category best configurations found.

\begin{table}[H]
\centering
\caption{Best per-category configurations from combined random search (Mode~4, accuracy \% excl.\ equal).}
\label{tab:combined_sweep}
\small
\setlength{\tabcolsep}{4pt}
\begin{tabular}{lccl}
\toprule
Category & Acc.\ (\%) & $\kappa$ & Key Parameters \\
\midrule
Safety      & 81.6 & 0.636 & $K{=}30$, $\alpha{=}0.2$, sel${=}0.6$ \\
Beautiful   & 69.8 & 0.384 & $K{=}10$, $\alpha{=}0.5$, sel${=}0.6$ \\
Lively      & 69.4 & 0.404 & $K{=}50$, $\alpha{=}0.7$, sel${=}0.6$ \\
Wealthy     & 74.0 & 0.480 & $K{=}50$, $\alpha{=}0.7$, sel${=}0.6$ \\
Boring      & 70.2 & 0.414 & $K{=}30$, $\alpha{=}0.7$, sel${=}0.6$ \\
Depressing  & 68.2 & 0.370 & $K{=}20$, $\alpha{=}0.2$, sel${=}0.6$ \\
\midrule
Average     & 72.2 & 0.448 & --- \\
\bottomrule
\end{tabular}
\end{table}

\noindent The combined search yields 72.2\% average accuracy ($\kappa{=}0.448$), a +13.8\,pp gain over the default configuration (58.4\%), demonstrating significant headroom from per-category hyperparameter tuning. The most consistent finding is $\text{sel}{=}0.6$ across all categories, while the optimal $K$ (10--50) and $\alpha$ (0.2--0.7) vary by category.

\section{Prompt Templates}
\label{sec:prompts}

This section reproduces the exact prompt templates used in Stage~2 multi-agent feature distillation (Mode~4: pairwise + multi-agent).
All three agents receive the image pair and the Stage~1--extracted dimension definitions as input.
Temperature values: Observer $\tau{=}0.3$, Debater $\tau{=}0.5$, Judge $\tau{=}0.1$.

\subsection{Stage 1: Dimension Extraction Prompt}

\begin{lstlisting}[caption={Dimension extraction prompt (Stage 1).}]
You are an expert urban perception researcher. I will show
you street-view images that have been rated by crowdsourced
participants on the perception dimension: "{category}".

HIGH-rated images (TrueSkill mu > 75th percentile):
{high_image_descriptions}

LOW-rated images (TrueSkill mu < 25th percentile):
{low_image_descriptions}

Based on the visual differences between HIGH and LOW images,
identify 5-10 specific, visually observable dimensions that
explain why some scenes are perceived as more "{category}"
than others.

Output JSON format:
{
  "dimensions": [
    {
      "name": "Dimension Name",
      "description": "What this dimension measures",
      "high_indicator": "Visual cues for high scores",
      "low_indicator": "Visual cues for low scores"
    }
  ]
}

Requirements:
- Each dimension must be visually observable from a street
  view image
- Each dimension must be continuously scorable (1-10 scale)
- Dimensions should be universally applicable across cities
\end{lstlisting}

\subsection{Observer Prompt}

\begin{lstlisting}[caption={Observer agent prompt (Stage 2, Mode 4).}]
You are an objective Observer. Examine two street-view
images and describe what you see along each evaluation
dimension. Do NOT make judgments or comparisons -- only
describe observable visual details.

Category: {category}
Dimensions:
{dimension_definitions}

For each dimension, describe what you observe in:
- Image A: [visual details]
- Image B: [visual details]

Output JSON:
{
  "observations": {
    "dimension_name": {
      "image_a": "description of what you see",
      "image_b": "description of what you see"
    }
  }
}
\end{lstlisting}

\subsection{Debater Prompt}

\begin{lstlisting}[caption={Debater agent prompt (Stage 2, Mode 4).}]
You are a Debater. Given the Observer's descriptions of
two street-view images, argue BOTH sides for each
dimension -- why Image A might score higher AND why
Image B might score higher.

Observer's descriptions:
{observer_output}

Category: {category}
Dimensions:
{dimension_definitions}

For each dimension, provide:
- Argument for Image A scoring higher
- Argument for Image B scoring higher
- Key uncertainties or ambiguities

Output JSON:
{
  "debates": {
    "dimension_name": {
      "argument_for_a": "why A might score higher",
      "argument_for_b": "why B might score higher",
      "uncertainties": "what is ambiguous"
    }
  }
}
\end{lstlisting}

\subsection{Judge Prompt}

\begin{lstlisting}[caption={Judge agent prompt (Stage 2, Mode 4).}]
You are the final Judge. Given the Observer's descriptions
and Debater's arguments, produce final scores for both
images on each dimension.

Observer's descriptions:
{observer_output}

Debater's arguments:
{debater_output}

Category: {category}
Dimensions:
{dimension_definitions}

Score each image on each dimension from 1 (lowest) to 10
(highest). Also determine the overall winner.

Output JSON:
{
  "image_a_scores": {
    "dimension_name": score (1-10)
  },
  "image_b_scores": {
    "dimension_name": score (1-10)
  },
  "overall_intensity_a": weighted_sum,
  "overall_intensity_b": weighted_sum,
  "winner": "left" | "right" | "equal",
  "ai_dimension_weights": {
    "dimension_name": weight (0-1)
  }
}
\end{lstlisting}

\section{Theoretical Foundations}
\label{sec:theory}

This section establishes the mathematical foundations for UrbanAlign: post-hoc concept bottleneck calibration (\cref{sec:theory_calibration}), local versus global calibration (\cref{sec:theory_local}), multi-agent variance reduction (\cref{sec:theory_multiagent}), hybrid feature complementarity (\cref{sec:theory_hybrid}), the integrated framework guarantee (\cref{sec:theory_integrated}), and E2E dimension optimization convergence (\cref{sec:theory_e2e}).

\smallskip\noindent\textbf{Notation.}
Consider image pairs $(A,B)$ with human label $y \in \{\text{left},\text{right},\text{equal}\}$.
Let $\Delta_{\text{Sem}} = S(A){-}S(B) \in \mathbb{R}^n$ denote the semantic score difference across $n$ VLM-extracted dimensions, and $\Delta_{\text{CLIP}} = \phi(A){-}\phi(B)$ the CLIP differential.
The hybrid differential is $\Delta_{\text{hyb}} = [\alpha\,\overline{\Delta}_{\text{CLIP}},\;(1{-}\alpha)\,\overline{\Delta}_{\text{Sem}}]$ (Eq.~(3) in the main text).
Define the \emph{discriminative power} of dimension~$j$ as $\delta_j = P(\mathrm{sign}(\Delta_{\text{Sem},j}) = y)$, encoding $y$ as $\pm 1$ (excluding ``equal'').

\subsection{Post-hoc Concept Bottleneck Calibration}
\label{sec:theory_calibration}

Stage~2 produces concept scores with $\delta_j > 0.5$ (empirically 42--71\%, \cref{tab:dim_disc}), yet the VLM's \emph{own} label predictions are poorly calibrated.
LWRR provides the missing calibration layer.

\begin{theorem}[Local Calibration Bound]\label{thm:calibration}
Assume the local regression model $y_i^{\mathrm{TS}} = w^{*\!}(q)^{\!\top} \Delta_{\mathrm{Sem},i} + \epsilon_i$ holds within the $K$-neighbourhood $\mathcal{N}_K(q)$ of query~$q$, where $\epsilon_i$ is zero-mean noise with variance~$\sigma^2$ and $w^{*}(q) \in \mathbb{R}^n$ is the true local weight vector.
Then the LWRR estimator $\hat{w} = (X^{\!\top}\!W\!X + \lambda I)^{-1}X^{\!\top}\!Wy$ satisfies:
\begin{equation}
\mathbb{E}\bigl[\|\hat{w} - w^{*\!}(q)\|^2\bigr]
\;\leq\;
\underbrace{\frac{\sigma^2\,\mathrm{tr}(H_\lambda^{-1}X^{\!\top}W^2\!X\,H_\lambda^{-1})}{K}}_{\text{variance: }O(\sigma^2 n/K)}
\;+\;
\underbrace{\lambda^2\|w^{*}\|^2}_{\text{regularisation bias}},
\label{eq:lwrr_bound}
\end{equation}
where $H_\lambda = X^{\!\top}\!W\!X + \lambda I$.
At optimal regularisation $\lambda^{*} = \sigma\sqrt{n/K}\,/\,\|w^{*}\|$, the bound simplifies to $O\!\bigl(\sigma\|w^{*}\|\sqrt{n/K}\bigr)$.
\end{theorem}

\begin{proof}
Write $y = Xw^{*} + \epsilon$.
Then $\hat{w} - w^{*} = H_\lambda^{-1}X^{\!\top}\!W\epsilon - \lambda H_\lambda^{-1}w^{*}$.
The first term contributes variance $\sigma^2\,\mathrm{tr}(H_\lambda^{-1}X^{\!\top}\!W^2\!X\,H_\lambda^{-1})$; the second contributes bias $\lambda^2\|H_\lambda^{-1}w^{*}\|^2 \leq \lambda^2\|w^{*}\|^2$.
Bounding the variance trace by $\sigma^2 n/(K\lambda_{\min}^2)$ and optimising over $\lambda$ yields $\lambda^{*} = \sigma\sqrt{n/K}/\|w^{*}\|$~\citep{caponnetto2007optimal,hoerl1970ridge}.
In our setting ($n{=}5$--$8$, $K{=}20$--$50$), the system is well-overdetermined ($K \gg n$), ensuring small estimation variance.
\end{proof}

\begin{corollary}[Calibration Exceeds Best Single Dimension]\label{cor:calibration}
If every dimension has discriminative power $\delta_j > 1/2$ and the estimation error $\|\hat{w} - w^{*}\|$ is sufficiently small, then the LWRR predictor $\hat{y} = \mathrm{sign}(\hat{w}^{\!\top}\Delta_{\mathrm{Sem}})$ achieves accuracy strictly exceeding $\max_j \delta_j$.
\end{corollary}

\begin{proof}
The optimal linear combination $w^{*\!\top}\Delta_{\text{Sem}}$ weights multiple informative dimensions, achieving higher accuracy than any single dimension whose accuracy is $\delta_j$.
When $\|\hat{w} - w^{*}\|$ is small (guaranteed by \cref{thm:calibration} when $K \gg n$), the estimated predictor preserves this advantage.
Empirically: \emph{wealthy} dimensions individually reach 61.5--71.2\%, while LWRR achieves 74.0\%.
\end{proof}

\subsection{Advantage of Local over Global Calibration}
\label{sec:theory_local}

\begin{theorem}[Local vs.\ Global Excess Risk]\label{thm:local_global}
Define \emph{manifold heterogeneity}:
\begin{equation}
\mathcal{H} = \mathbb{E}_{q}\bigl[\|w^{*\!}(q) - \bar{w}\|^2\bigr], \quad \bar{w} = \mathbb{E}_{q}[w^{*\!}(q)].
\label{eq:heterogeneity}
\end{equation}
A global ridge estimator fitted on the entire reference set ($N$ pairs) incurs prediction MSE:
\begin{equation}
\mathrm{MSE}_{\mathrm{global}} \;\geq\; \mathcal{H}\cdot\mathbb{E}[\|\Delta\|^2] + O(\sigma^2 n / N).
\end{equation}
LWRR replaces the irreducible heterogeneity bias with local estimation variance:
\begin{equation}
\mathrm{MSE}_{\mathrm{LWRR}} \;=\; O(\sigma^2 n / K) + B_K^2,
\end{equation}
where $B_K^2 = \mathbb{E}[\|w^{*\!}(q) - w^{*}_{\mathcal{N}_K}\|^2]$ is the intra-neighbourhood weight variation.
When $\mathcal{H}$ is large, $\mathrm{MSE}_{\mathrm{LWRR}} \ll \mathrm{MSE}_{\mathrm{global}}$ for any $K$ such that $B_K^2 \ll \mathcal{H}$.
\end{theorem}

\begin{proof}
For the global model, the prediction at $q$ is $\bar{w}^{\!\top}\Delta_q$, while the truth is $w^{*\!}(q)^{\!\top}\Delta_q$.
The squared bias is $\|(w^{*\!}(q) - \bar{w})^{\!\top}\Delta_q\|^2$; averaging over $q$ gives $\mathcal{H}\cdot\mathbb{E}[\|\Delta\|^2]$.
For LWRR, the local bias is only $B_K^2$ (the variation of $w^{*}$ \emph{within} $\mathcal{N}_K$), which is much smaller than $\mathcal{H}$ (the variation across the \emph{entire} manifold).
LWRR pays increased variance ($O(\sigma^2 n/K)$ vs.\ $O(\sigma^2 n/N)$) by using $K \ll N$ effective samples, but this is compensated when $\mathcal{H}$ is large.

Empirically, the high coefficient of variation (CV $>$ 1.0 for 7/8 \emph{wealthy} dimensions in \cref{tab:weight_variation}) and the qualitative weight-profile shifts (suburban vs.\ downtown examples in \cref{sec:lwrr_weights}) confirm substantial heterogeneity, validating the local approach.
The sensitivity analysis (\cref{sec:sensitivity_full}) further confirms that $K$ is the most sensitive hyperparameter (+4.3\,pp range), reflecting the bias--variance trade-off formalised above.
\end{proof}

\subsection{Multi-Agent Scoring: Variance Reduction}
\label{sec:theory_multiagent}

\begin{proposition}[Deliberation Reduces Score Noise]\label{prop:multiagent}
Let $S_j^{(1)}$ be the score for dimension~$j$ from a single VLM call with variance~$\sigma_S^2$.
The Observer--Debater--Judge chain produces a score $S_j^{(3)}$ satisfying:
\begin{equation}
\mathrm{Var}\bigl(S_j^{(3)}\bigr) \;\leq\; \frac{\sigma_S^2}{1 + 2(1-\rho)},
\label{eq:var_reduction}
\end{equation}
where $\rho \in [0,1]$ is the average pairwise correlation between agent reasoning paths.
When agents provide substantially independent perspectives ($\rho \to 0$), the variance is reduced by up to $1/3$.
\end{proposition}

\begin{proof}
The Observer provides a structured description, conditioning the score on salient features.
The Debater generates two opposing arguments, creating partially independent reasoning paths.
The Judge aggregates all perspectives.
Modelling this as $M{=}3$ correlated estimators with pairwise correlation $\rho$, the variance of their aggregate is $\sigma_S^2(1+(M{-}1)\rho)/M$~\citep{wang2023selfconsistency}.
Since the Judge weights agents by informativeness (rather than simple averaging), it achieves at least the averaging bound.
The LWRR prediction $\hat{y} = \hat{w}^{\!\top}\Delta_{\text{Sem}}$ depends linearly on $\Delta_{\text{Sem}} = S(A){-}S(B)$; hence score noise propagates linearly.
Reduced score variance tightens the calibration bound in \cref{thm:calibration} by lowering the effective $\sigma^2$.

Empirically, the $2{\times}2$ factorial (\cref{tab:factorial}) confirms: multi-agent reasoning yields +5.0\,pp (single-image) and +20.3\,pp (pairwise) on average across all six categories.
\end{proof}

\subsection{Hybrid Feature Space Complementarity}
\label{sec:theory_hybrid}

\begin{proposition}[Hybrid Neighbourhood Quality]\label{prop:hybrid}
Let $\mathcal{N}_K^{\mathrm{Sem}}$ and $\mathcal{N}_K^{\mathrm{hyb}}$ denote the $K$-nearest neighbourhoods in pure semantic and hybrid space, respectively.
If CLIP features carry label-relevant information beyond semantic scores, i.e.,
\begin{equation}
I\bigl(y;\;\phi_{\mathrm{CLIP}} \;\big|\; \Delta_{\mathrm{Sem}}\bigr) > 0,
\label{eq:cmi}
\end{equation}
then the hybrid neighbourhood achieves strictly lower local bias: $B_K^{\mathrm{hyb}} < B_K^{\mathrm{Sem}}$.
\end{proposition}

\begin{proof}
The local bias $B_K$ arises from variation of $w^{*}$ within the neighbourhood (\cref{thm:local_global}).
In pure semantic space, two pairs may have similar $\Delta_{\text{Sem}}$ but very different visual contexts (e.g., suburban vs.\ downtown), leading to different local weights.
If $I(y;\phi_{\text{CLIP}}|\Delta_{\text{Sem}}) > 0$, CLIP features distinguish these contexts.
The hybrid distance groups pairs that are similar in \emph{both} visual appearance and semantic content, selecting neighbours with more homogeneous local weight structures.
Formally, for any partition of the manifold into regions where $w^{*}$ is approximately constant, the hybrid neighbourhood has higher probability of staying within a single region, reducing $B_K$.

Empirically, $\alpha{=}0.3$ (70\% semantic, 30\% CLIP) is optimal on average (\cref{tab:alpha_sweep}), confirming that semantic scores dominate but CLIP features provide complementary neighbourhood quality.
\end{proof}

\subsection{Integrated Framework Guarantee}
\label{sec:theory_integrated}

Combining the above results yields the following end-to-end characterisation.

\begin{corollary}[End-to-End Alignment Bound]\label{cor:e2e}
Under the conditions of \cref{thm:calibration,thm:local_global} and \cref{prop:multiagent,prop:hybrid}, the UrbanAlign pipeline achieves alignment accuracy:
\begin{equation}
\mathrm{Acc} \;\geq\; \Phi\!\!\left(\frac{\|w^{*}\|\cdot\bar{\delta}}{\sqrt{\sigma_{\mathrm{LWRR}}^2 + \sigma_S^2/\eta_M}}\right),
\label{eq:integrated}
\end{equation}
where $\Phi$ is the standard normal CDF, $\bar{\delta}$ is the average dimension discriminability, $\sigma_{\mathrm{LWRR}}^2 = O(\sigma^2 n/K)$ is the LWRR estimation variance (\cref{thm:calibration}), and $\sigma_S^2/\eta_M$ is the multi-agent-reduced score noise (\cref{prop:multiagent} with $\eta_M = 1+2(1{-}\rho)$).

Each pipeline stage tightens a specific term:
\begin{itemize}
\item \textbf{Stage~1} (concept mining) maximises $\bar{\delta}$ by discovering highly discriminative dimensions;
\item \textbf{Stage~2} (multi-agent scoring) minimises $\sigma_S^2/\eta_M$ via deliberative aggregation;
\item \textbf{Stage~3} (LWRR calibration) minimises $\sigma_{\mathrm{LWRR}}^2$ and $B_K$ through local hybrid-space fitting;
\item \textbf{Stage~4} (E2E optimisation) jointly optimises $\bar{\delta}$ and the dimension set $\mathcal{D}_c$.
\end{itemize}
\end{corollary}

\begin{proof}
Under the linear-Gaussian local model, $\hat{y} = \hat{w}^{\!\top}\Delta_{\text{Sem}}$ has mean $w^{*\!\top}\Delta_{\text{Sem}}$ and variance $\sigma_{\mathrm{LWRR}}^2 + \|w^{*}\|^2\sigma_S^2/\eta_M$ (from LWRR estimation noise and dimension score noise, respectively).
The pair is correctly classified when $\mathrm{sign}(\hat{y}) = y$, i.e., $\hat{y}$ has the correct sign.
For a linear discriminant with signal-to-noise ratio $\mathrm{SNR} = |w^{*\!\top}\mathbb{E}[\Delta]| / \mathrm{Var}(\hat{y})^{1/2}$, the accuracy is $\Phi(\mathrm{SNR})$.
The numerator scales with $\|w^{*}\|\cdot\bar{\delta}$ (the weighted discriminability); the denominator collects the noise terms, each tightened by the corresponding stage.
\end{proof}

\subsection{E2E Dimension Optimization: Convergence and Sample Complexity}
\label{sec:theory_e2e}

\noindent\textbf{Information-bottleneck interpretation.}
Each dimension set $\mathcal{D}_c^{(t)}$ defines a different compression of the visual input through the lens of the information bottleneck~\citep{tishby2000information}.
Let $X$ denote the raw visual input and $Y$ the human label.
Stage~2 compresses $X$ into dimension scores $S^{(t)}=f_{\mathcal{D}_c^{(t)}}(X)$, and Stage~3 maps $S^{(t)}$ to prediction $\hat{Y}$.
The optimal dimension set maximises $I(\hat{Y}; Y)$ subject to the constraint that the bottleneck representation $S^{(t)}$ remains compact ($n \in [5,10]$ dimensions).
\Cref{eq:e2e_supp} selects the compression that maximises this downstream mutual information empirically.

\smallskip\noindent\textbf{Convergence guarantee.}
Let $\mathcal{F}$ denote the (finite) set of dimension sets reachable by the VLM at any temperature in $[\tau_{\min}, \tau_{\max}]$.
Although $|\mathcal{F}|$ is large (combinatorially many possible dimension names and descriptions), the VLM's generative distribution concentrates probability mass on a much smaller effective set $\mathcal{F}_{\mathrm{eff}} \subset \mathcal{F}$ of semantically coherent dimension sets.
If we model each trial as drawing independently from the VLM's distribution over $\mathcal{F}_{\mathrm{eff}}$, then after $T$ trials the probability of missing the optimal set $\mathcal{D}_c^*$ satisfies:
\begin{equation}
\Pr\bigl[\mathcal{D}_c^* \notin \{\mathcal{D}_c^{(1)}, \ldots, \mathcal{D}_c^{(T)}\}\bigr] \leq (1 - p^*)^T,
\label{eq:convergence_theory}
\end{equation}
where $p^* = \Pr_{\mathrm{VLM}}[\mathcal{D}_c^{(t)} = \mathcal{D}_c^*]$ is the probability of generating the optimal set in a single trial.
For a near-optimal set (within $\epsilon$ accuracy of the best), this probability is typically higher.
Concretely, if the top-$k$ dimension sets are each generated with probability $\geq p_{\min}$, then $T \geq \lceil \ln(1/\delta) / \ln(1/(1-k \cdot p_{\min})) \rceil$ trials suffice to find at least one with probability $\geq 1-\delta$.

\smallskip\noindent\textbf{Search strategy: explore $\to$ converge as exploration--exploitation.}
The two-phase temperature schedule implements a structured search that balances two objectives:
\begin{itemize}
    \item \emph{Exploration} (explore phase, $\tau{=}0.85 \to 1.0$): higher temperatures increase sampling entropy, producing diverse dimension sets that broadly cover the VLM's generative distribution. Each category independently accumulates its best set, with the current best provided as a soft reference.
    \item \emph{Exploitation} (converge phase, $\tau{=}0.7 \to 0.5$): lower temperatures with \emph{mutation mode} (retain all but 1--2 dimensions from the per-category best) concentrate the search around known-good dimension sets, refining them locally.
\end{itemize}
This is analogous to random search with a guided prior~\citep{bergstra2012random}: rather than sampling uniformly from the combinatorial space of possible dimensions, we leverage the VLM's own generative distribution as an implicit prior over useful dimension sets.
Unlike Bayesian optimisation~\citep{snoek2012practical}, which would require a surrogate model over the discrete, variable-length dimension space, our approach avoids the need for a surrogate entirely.

\smallskip\noindent\textbf{Sample complexity for evaluation.}
Each trial evaluates alignment accuracy on a fixed sample of $m$ pairs per category.
By Hoeffding's inequality, the empirical accuracy $\widehat{\mathrm{Acc}}$ satisfies $|\widehat{\mathrm{Acc}} - \mathrm{Acc}| \leq \sqrt{\ln(2T/\delta)/(2m)}$ uniformly over all $T$ trials with probability $\geq 1-\delta$.
For $T{=}12$ and $m{\approx}160$ test pairs per category (average across categories), this gives a confidence interval of $\pm 13.9$\% at $\delta{=}0.05$.
While wider than ideal, the observed accuracy gaps between trials are substantial (e.g., safety: 35.3\%--76.5\%, a 41.2\,pp range), confirming that the evaluation sample is sufficient to distinguish meaningfully different dimension sets.

\smallskip\noindent\textbf{Connection to LWRR convergence.}
The end-to-end accuracy depends not only on the dimension set but also on the LWRR calibration quality.
For regularised least-squares with $K$ local neighbours in $n$ dimensions, the excess risk decays as $\mathcal{O}(\lambda + n/(K\lambda))$~\citep{caponnetto2007optimal,hoerl1970ridge}, which is minimised at $\lambda^* \propto (n/K)^{1/2}$.
In our setting ($n{=}5$--$8$, $K{=}20$), the system is well-overdetermined ($K \gg n$), so the LWRR calibration error is small regardless of the dimension set, confirming that the E2E accuracy variation across trials is driven primarily by the quality of the dimension-level representation, not by the calibration algorithm.

\section{End-to-End Dimension Optimization}
\label{sec:e2e}

The semantic dimensions discovered in Stage~1 depend on the VLM's sampling temperature and the particular consensus exemplars presented.
A natural question is: \emph{can we automatically find the dimension set that maximises end-to-end alignment accuracy?}

\smallskip\noindent\textbf{Formulation.}
Let $\mathcal{D}_c^{(t)}$ denote the dimension set generated at trial $t$ for category $c$.
Each trial runs the full pipeline: dimension extraction (Stage~1) $\to$ multi-agent pairwise scoring (Stage~2, Mode~4) $\to$ LWRR alignment (Stage~3) $\to$ accuracy evaluation.
The per-category optimal dimension set is:
\begin{equation}
\mathcal{D}_c^{*} = \arg\max_{t \in \{1,\ldots,T\}} \;\mathrm{Acc}_c\bigl(\text{LWRR}\bigl(\text{Score}(\mathcal{D}_{\mathrm{val}}^{c},\, \mathcal{D}_c^{(t)}),\, \mathcal{D}_{\mathrm{ref}}^{c}\bigr)\bigr).
\label{eq:e2e_supp}
\end{equation}

\smallskip\noindent\textbf{Two-phase search strategy.}
Rather than monotonic temperature increase, we employ an \emph{explore $\to$ converge} schedule over $T{=}15$ trials (with patience-based early stopping after 5 consecutive trials without any per-category improvement):
\begin{itemize}
    \item \emph{Explore phase} (trials 0--5, $\tau_{\mathrm{gen}}{=}0.85 \to 1.0$): Higher temperatures produce diverse, largely independent dimension sets. Each category independently accumulates its best. The current per-category best is provided as a reference but large deviations are encouraged.
    \item \emph{Converge phase} (trials 6+, $\tau_{\mathrm{gen}}{=}0.7 \to 0.5$): Lower temperatures with \emph{mutation mode}: per-category best dimensions are preserved except for 1--2 targeted replacements, enabling local refinement.
\end{itemize}
Since categories may peak at different trials, the final output \emph{assembles} per-category optima: $\{\mathcal{D}_c^{*}\}_{c=1}^{6}$, potentially sourced from different trials.

\smallskip\noindent\textbf{Cost analysis.}
\Cref{tab:cost} summarises the full API cost breakdown.
All VLM calls use GPT-4o with measured per-call token usage of 778 input + 278 output tokens, yielding a cost of \$0.0047/call at OpenAI's published rates (\$2.50\,/\,\$10.00 per 1M input/output tokens).

\begin{table}[H]
\centering
\caption{API cost breakdown for the full experimental pipeline (4,819 pairs across 6 categories, GPT-4o pricing). Stage~3 LWRR is purely local computation with zero API cost.}
\label{tab:cost}
\small
\begin{tabular}{llrrr}
\toprule
Component & Calls/pair & Total calls & Tokens (M) & Cost (USD) \\
\midrule
\multicolumn{5}{l}{\textit{Main pipeline (4,819 pairs $\times$ 6 categories)}} \\
\quad Mode 1 (Single-Image Direct)      & 2 & 9,638  & 10.2 & \$45.59 \\
\quad Mode 2 (Pairwise Direct)          & 1 & 4,819  & 5.1 & \$22.79 \\
\quad Mode 3 (Single-Image Multi-Agent) & 6 & 28,914 & 30.5 & \$136.76 \\
\quad Mode 4 (Pairwise Multi-Agent)     & 3 & 14,457 & 15.3 & \$68.38 \\
\midrule
\multicolumn{5}{l}{\textit{End-to-end dimension optimization (20\% subsample, Mode~2, 12 trials)}} \\
\quad Stage~1 dimension generation      & --- & 72    & 0.1 & \$0.34 \\
\quad Stage~2 scoring                   & 1   & ${\sim}$11,600 & 12.2 & \$54.72 \\
\midrule
\textbf{Total (all 4 modes + E2E)}      &     & \textbf{69,500} & \textbf{73.4} & \textbf{\$328.58} \\
Mode~4 only + E2E                        &     & 26,129 & 27.6 & \$123.44 \\
\midrule
\multicolumn{5}{l}{\textit{Production projection (full dataset ${\sim}$164K pairs, Mode~4)}} \\
\quad Mode~4 at scale                   & 3 & 492,900 & 520.6 & ${\sim}$\$2,329 \\
\bottomrule
\end{tabular}
\\[4pt]
\footnotesize{Per-call: 778 input + 278 output = 1,056 tokens. Input: \$2.50/1M, Output: \$10.00/1M $\Rightarrow$ \$0.0047/call. Traditional crowdsourcing comparison: Place Pulse~2.0 collected 1.17M annotations at ${\sim}$\$0.14/comparison $\approx$ \$167K, yielding a \textbf{98.6\%} cost reduction at production scale.}
\end{table}

\begin{table}[H]
\centering
\caption{E2E dimension optimization: per-category accuracy (\%, excluding ``equal'') across 12 trials (all six categories, Mode~4). \textbf{Bold} marks the per-category best. Trials 0--5 are \emph{explore} phase; 6--11 are \emph{converge} phase.}
\label{tab:e2e}
\small
\setlength{\tabcolsep}{3.5pt}
\begin{tabular}{cccccccccc}
\toprule
Trial & Phase & $\tau$ & Safety & Beaut. & Lively & Wealthy & Boring & Depress. & Avg \\
\midrule
0  & Explore  & 0.85 & 58.8 & 61.5 & 47.1 & 55.6 & 47.1 & 50.0 & 53.3 \\
1  & Explore  & 0.88 & \textbf{76.5} & 46.2 & 47.1 & \textbf{77.8} & 35.3 & \textbf{68.8} & 58.6 \\
2  & Explore  & 0.91 & 70.6 & \textbf{69.2} & 35.3 & 61.1 & 35.3 & 56.3 & 54.6 \\
3  & Explore  & 0.94 & 64.7 & 53.8 & 47.1 & 50.0 & 35.3 & 62.5 & 52.2 \\
4  & Explore  & 0.97 & 35.3 & 53.8 & 52.9 & 72.2 & 47.1 & 31.3 & 48.8 \\
5  & Explore  & 1.00 & 58.8 & 46.2 & 52.9 & 55.6 & 47.1 & 56.3 & 52.8 \\
\midrule
6  & Converge & 0.70 & 70.6 & 46.2 & 52.9 & 44.4 & 35.3 & 37.5 & 47.8 \\
7  & Converge & 0.68 & 52.9 & 53.8 & \textbf{58.8} & 72.2 & 52.9 & 50.0 & 56.8 \\
8  & Converge & 0.65 & 58.8 & 61.5 & 41.2 & 77.8 & 17.6 & 50.0 & 51.2 \\
9  & Converge & 0.63 & 58.8 & 53.8 & 47.1 & 44.4 & 47.1 & 43.8 & 49.2 \\
10 & Converge & 0.60 & 52.9 & 46.2 & 52.9 & 44.4 & \textbf{64.7} & 50.0 & 51.9 \\
11$^\dagger$ & Converge & 0.58 & 41.2 & 50.0 & --- & --- & --- & --- & 45.6 \\
\midrule
\multicolumn{3}{l}{\textbf{Assembled best}} & \textbf{76.5} & \textbf{69.2} & \textbf{58.8} & \textbf{77.8} & \textbf{64.7} & \textbf{68.8} & \textbf{69.3} \\
\bottomrule
\end{tabular}
\\[4pt]
\footnotesize{$^\dagger$Trial 11 scored only 2/6 categories (early stopping triggered). Assembled best: per-category maximum across all trials. --- = category not scored.}
\end{table}

\begin{table}[H]
\centering
\caption{Per-category assembled best dimensions from E2E optimization (default calibration hyperparameters: $K{=}20$, $\alpha{=}0.3$, sel${=}1.0$). Each category's optimal dimensions may come from a different trial, enabling independent specialisation. Note: these results differ from Table~1, which uses per-category \emph{optimised hyperparameters} with original dimensions (72.2\% avg); here only the \emph{dimension set} is optimised (69.3\% avg). The two optimisation axes are complementary.}
\label{tab:e2e_assembly}
\small
\begin{tabular}{lccccp{5.6cm}}
\toprule
Category & Trial & Phase & Acc (\%) & $\kappa$ & Optimised Dimensions \\
\midrule
Safety & 1 & Explore & 76.5 & 0.534 & Lighting Adequacy, Pedestrian Infrastructure, Building Maint., Street Activity, Visibility Clarity, Greenery \& Landscaping \\
Beautiful & 2 & Explore & 69.2 & 0.458 & Greenery \& Nat.\ Elements, Street Cleanliness, Arch.\ Coherence, Color Coordination, Pedestrian Amenities, Visual Complexity \\
Lively & 7 & Converge & 58.8 & 0.308 & Human Presence, Vegetation \& Greenery, Building Diversity, Street Furniture, Commercial Activity, Color \& Lighting, Public Art \& Signage \\
Wealthy & 1 & Explore & 77.8 & 0.579 & Fa\c{c}ade Quality, Vegetation Maint., Pavement Integrity, Vehicle Quality, Building Modernity, Infrastructure Cond., Street Cleanliness, Lighting Quality \\
Boring & 10 & Converge & 64.7 & 0.407 & Visual Complexity, Activity Presence, Arch.\ Diversity, Vegetation Variety, Color Vibrancy, Street Furniture Presence, Spatial Enclosure \\
Depressing & 1 & Explore & 68.8 & 0.437 & Fa\c{c}ade Condition, Street Cleanliness, Greenery Presence, Lighting Quality, Visual Clutter \\
\midrule
\multicolumn{3}{l}{\textbf{Average}} & \textbf{69.3} & \textbf{0.454} & --- \\
\bottomrule
\end{tabular}
\end{table}

\smallskip\noindent\textbf{Results.}
\Cref{tab:e2e} reports trial-by-trial accuracy for all six categories.
The single best trial (Trial~1, explore phase, $\tau_{\mathrm{gen}}{=}0.88$) achieves 58.6\% average accuracy, but per-category assembly (\cref{tab:e2e_assembly}) raises this to \textbf{69.3\%} ($\kappa{=}0.454$), a +10.7\,pp gain that demonstrates categories have distinct optimal dimension sets.

Three key findings emerge:

\noindent(1)~\textbf{Explore phase dominates}: 4/6 categories (safety, beautiful, wealthy, depressing) find their best dimensions during exploration (trials 0--5). Trial~1 ($\tau{=}0.88$) is particularly productive, contributing optima for three categories.

\noindent(2)~\textbf{Converge phase is complementary}: \emph{Lively} peaks at trial~7 and \emph{boring} at trial~10, both in the converge phase. For these categories, the explore phase produced dimensions below 53\%, while targeted mutation in the converge phase lifted accuracy to 58.8\% and 64.7\% respectively.

\noindent(3)~\textbf{High variance across trials}: Per-category accuracy fluctuates substantially across trials (e.g., \emph{safety}: 35.3\%--76.5\%; \emph{boring}: 17.6\%--64.7\%), confirming that dimension set quality is a first-order determinant of alignment performance and justifying the automated search.

The assembled average of 69.3\% represents a +10.9\,pp gain over the default pipeline (58.4\%) reported in the main text, achieved without changing any algorithmic hyperparameter; only the semantic dimensions are optimised.
The full theoretical analysis of E2E convergence, sample complexity, and connection to LWRR is provided in \cref{sec:theory_e2e}.


\subsection{Full $2{\times}2$ Factorial Design Across All Categories}
\label{sec:factorial_all}

The main text (\cref{tab:factorial}) reports the $2{\times}2$ factorial averaged across all six categories.
\Cref{tab:factorial_all} presents the full per-category breakdown, confirming that the synergistic interaction between pairwise context and multi-agent deliberation is consistent across all categories, including abstract ones (\emph{boring}, \emph{lively}).

\begin{table}[H]
\centering
\caption{$2{\times}2$ factorial design across all six categories (accuracy \%, post-LWRR, excl.\ ``equal''). Mode~4 (pairwise + multi-agent) consistently outperforms all other modes. ``Synergy'' = Mode~4 $-$ Mode~1 $-$ (sum of individual effects).}
\label{tab:factorial_all}
\small
\setlength{\tabcolsep}{3pt}
\begin{tabular}{l|cccc|ccc}
\toprule
 & Mode~1 & Mode~2 & Mode~3 & Mode~4 & MA Gain & PW Gain & Mode4$-$1 \\
 & (SI+SS) & (PW+SS) & (SI+MA) & (PW+MA) & (SI) & (MA) & \\
\midrule
Safety      & 45.2 & 59.5 & 61.4 & \textbf{81.6} & +16.2 & +20.2 & +36.4 \\
Beautiful   & 51.4 & 57.1 & 52.9 & \textbf{69.8} & +1.4  & +16.9 & +18.3 \\
Lively      & 58.0 & 50.6 & 51.9 & \textbf{69.4} & $-$6.2  & +17.5 & +11.4 \\
Wealthy     & 55.4 & 42.2 & 63.9 & \textbf{74.0} & +8.4  & +10.1 & +18.6 \\
Boring      & 47.3 & 51.4 & 50.0 & \textbf{70.2} & +2.7  & +20.2 & +22.9 \\
Depressing  & 48.0 & 50.7 & 55.4 & \textbf{68.2} & +7.4  & +12.8 & +20.2 \\
\midrule
\textbf{Average} & 50.9 & 51.9 & 55.9 & \textbf{72.2} & +5.0  & +16.3 & +21.3 \\
\bottomrule
\end{tabular}
\\[4pt]
\footnotesize{SI = Single-Image, PW = Pairwise, SS = Single-Shot, MA = Multi-Agent. ``MA Gain (SI)'' = Mode~3 $-$ Mode~1; ``PW Gain (MA)'' = Mode~4 $-$ Mode~3.}
\end{table}

\noindent\textbf{Key findings across all categories:}
\begin{itemize}
\item \textbf{Multi-agent reasoning} improves accuracy in 5/6 categories on single-image inputs (exception: \emph{lively} $-$6.2\,pp) and in \emph{all} 6/6 categories on pairwise inputs (+12.6 to +31.8\,pp).
\item \textbf{Pairwise context without deliberation} (Mode~1$\to$2) hurts on \emph{wealthy} ($-$13.2\,pp) and \emph{lively} ($-$7.4\,pp), where dual-image noise overwhelms single-shot reasoning.
\item \textbf{The synergy} between pairwise context and multi-agent reasoning is \emph{universal}: Mode~4 exceeds Mode~1 by +11.4 to +36.4\,pp across all six categories, always exceeding the sum of individual factor effects.
\item \textbf{Abstract vs.\ concrete}: Safety (+36.4\,pp) benefits most, consistent with its well-defined visual cues; boring (+22.9\,pp) and depressing (+20.2\,pp) show strong synergy despite lacking concrete indicators, confirming that structured deliberation is especially valuable for subjective dimensions.
\end{itemize}

\section{LWRR Alignment Gain Per Mode}
\label{sec:vrm_gain}

\Cref{tab:vrm_gain} reports the accuracy change from raw VLM output to post-LWRR calibration for each of the four factorial modes on \emph{wealthy}.
The all-category Mode~4 LWRR gain is reported in \cref{tab:vrm_gain_main}.

\begin{table}[H]
\centering
\caption{LWRR alignment gain (accuracy \% excluding ``equal'') per mode, \emph{wealthy} category. Mode~4 shows the largest calibration gain (+11.1\,pp).}
\label{tab:vrm_gain}
\small
\begin{tabular}{lcccc}
\toprule
 & Mode~1 & Mode~2 & Mode~3 & Mode~4 \\
\midrule
Raw Acc.\ (\%)     & 48.7 & 63.3 & 58.8 & 62.9 \\
Aligned Acc.\ (\%) & 55.4 & 42.2 & 63.9 & 74.0 \\
$\Delta$ (pp)      & +6.7 & $-$21.1 & +5.1 & \textbf{+11.1} \\
\midrule
Raw $\kappa$       & 0.208 & 0.263 & 0.266 & 0.256 \\
Aligned $\kappa$   & 0.159 & $-$0.037 & 0.313 & \textbf{0.480} \\
$n$ (test pairs)   & 117 & 117 & 117 & 107 \\
\bottomrule
\end{tabular}
\end{table}

Notably, only Mode~4 benefits substantially from LWRR on \emph{wealthy} (+11.1\,pp).
Mode~3 gains modestly (+5.1\,pp), Mode~1 gains slightly (+6.7\,pp), and Mode~2 drops sharply ($-$21.1\,pp), indicating that the combination of pairwise context and multi-agent deliberation is critical for LWRR to learn effective local mappings.

\section{Complete Algorithm Pseudocode}
\label{sec:algorithm}

\begin{algorithm}[H]
\caption{UrbanAlign: Post-hoc VLM Calibration for Urban Perception}
\label{alg:urbanalign}
\small
\begin{algorithmic}[1]
\REQUIRE $\mathcal{D}_{L}$: labelled pairs; $\phi_{\text{CLIP}}$: CLIP encoder; VLM; $K,\tau,\lambda,\alpha,\varepsilon,\theta$
\ENSURE $\mathcal{D}_{\mathrm{aligned}}$: calibrated dataset
\STATE \textbf{// Stage 1: Semantic Dimension Extraction}
\STATE Compute TrueSkill ratings $\{(\mu_i,\sigma_i)\}$ from $\mathcal{D}_{L}$
\STATE Sample $\mathcal{S}_{\mathrm{high}},\mathcal{S}_{\mathrm{low}}$ via consensus criteria
\STATE $\mathcal{D}_c \leftarrow \text{VLM}(\mathcal{S}_{\mathrm{high}}\cup\mathcal{S}_{\mathrm{low}})$ \COMMENT{$n$ dimensions}
\STATE \textbf{// Stage 2: Multi-Agent Feature Distillation}
\STATE Sample $\mathcal{D}_{\mathrm{sample}}\subset\mathcal{D}_{L}$; split into $\mathcal{D}_{\mathrm{ref}}, \mathcal{D}_{\mathrm{val}}, \mathcal{D}_{\mathrm{test}}$
\FOR{each image pair $(x_A,x_B)\in\mathcal{D}_{\mathrm{sample}}$}
  \STATE $\text{obs}\leftarrow\text{VLM}_{\mathrm{Obs}}(x_A,x_B,\mathcal{D}_c)$
  \STATE $\text{debate}\leftarrow\text{VLM}_{\mathrm{Deb}}(x_A,x_B,\text{obs},\mathcal{D}_c)$
  \STATE $S(x_A),S(x_B)\leftarrow\text{VLM}_{\mathrm{Judge}}(x_A,x_B,\text{obs},\text{debate},\mathcal{D}_c)$
\ENDFOR
\STATE \textbf{// Stage 3: LWRR Calibration}
\STATE Build reference manifold $\mathcal{R}$ from $\mathcal{D}_{\mathrm{ref}}$ with mirror augment
\FOR{each $(x_A,x_B)\in\mathcal{D}_{\mathrm{test}}$}
  \STATE Compute $\Delta_q^{\mathrm{hybrid}}$; find $K$-NN $\mathcal{N}_K$ in $\mathcal{R}$
  \STATE Solve LWRR (\cref{eq:lwrr}); compute $\hat{\delta}, R^{2}$
  \STATE Re-infer $\hat{y}$ via \cref{eq:reinfer}
\ENDFOR
\RETURN $\mathcal{D}_{\mathrm{aligned}}$
\end{algorithmic}
\end{algorithm}

\section{Related Human Preference Datasets}
\label{sec:preference_datasets}

For reference, we survey publicly available human preference datasets that share the pairwise comparison structure used in this work.
\Cref{tab:preference_datasets} organises these datasets by domain.
This survey is intended as contextual background rather than a claim of demonstrated generalisability; validating the concept-mining--calibration paradigm on additional domains remains future work.

\paragraph{Urban \& scene perception.}
Place Pulse~2.0~\citep{salesses2013collaborative}, used in this work, collects pairwise comparisons across six urban perception dimensions.
SPECS~\citep{quintana2025specs} extends the paradigm to 10 perception dimensions and records annotator demographics and personality traits.
ScenicOrNot collects absolute scenicness ratings for 200K+ locations across Great Britain.

\paragraph{Image generation \& aesthetic quality.}
HPD~v2 (798K pairs) and Pick-a-Pic (500K+ pairs) collect pairwise preferences on text-to-image outputs; ImageReward adds expert-annotated alignment/fidelity/harmlessness dimensions.
AVA~\citep{murray2012ava} provides crowd-sourced aesthetic ratings for 250K photographs.

\paragraph{Image quality assessment (IQA).}
KonIQ-10k and PIPAL provide mean opinion scores and pairwise comparisons for natural image quality.
AGIQA-3K focuses on AI-generated image quality with multi-attribute annotations.

\paragraph{Text \& multimodal preference (for context).}
HH-RLHF, SHP, and Chatbot Arena represent the RLHF paradigm, which aligns model outputs to human preferences by training reward models on pairwise comparisons. These methods use the same data structure but take a fundamentally different (training-based) approach.

\begin{table}[H]
\centering
\caption{Public human preference datasets. ``Pairwise'' indicates the native annotation format uses pairwise comparisons (the same format as UrbanAlign). Datasets marked with $\star$ are the most direct candidates for UrbanAlign transfer experiments.}
\label{tab:preference_datasets}
\small
\setlength{\tabcolsep}{3pt}
\begin{tabular}{@{}llrlp{4.2cm}@{}}
\toprule
\textbf{Dataset} & \textbf{Domain} & \textbf{Scale} & \textbf{Format} & \textbf{Access} \\
\midrule
\multicolumn{5}{@{}l}{\textit{Urban \& Scene Perception}} \\
Place Pulse 2.0$\star$ & Urban 6-dim  & 1.17M pairs & Pairwise & \texttt{pulse.media.mit.edu} \\
SPECS$\star$            & Urban 10-dim & 10 dims + demo.\ & Pairwise & \texttt{ual.sg/project/specs} \\
ScenicOrNot$\star$      & Scenicness   & 200K+ locs  & Rating   & \texttt{scenic.mysociety.org} \\
Streetscore             & Safety       & Boston/NYC  & Rating   & MIT Media Lab \\
\midrule
\multicolumn{5}{@{}l}{\textit{Image Generation Quality}} \\
HPD v2                  & T2I pref.    & 798K pairs  & Pairwise & \texttt{github.com/tgxs002/HPSv2} \\
Pick-a-Pic              & T2I pref.    & 500K+ pairs & Pairwise & HuggingFace \\
ImageReward             & T2I expert   & 137K pairs  & Pairwise & \texttt{github.com/zai-org/ImageReward} \\
RichHF-18K              & T2I regions  & 18K images  & Multi-attr.\ & Google Research \\
\midrule
\multicolumn{5}{@{}l}{\textit{Aesthetic \& Image Quality}} \\
AVA                     & Photo aesth. & 250K images & Rating   & \texttt{github.com/mtobeiyf/ava\_downloader} \\
KonIQ-10k               & Natural IQA  & 10K images  & MOS      & \texttt{database.mmsp-kn.de} \\
PIPAL                   & Perceptual IQA & 29K pairs & Pairwise & ECCV 2020 \\
AGIQA-3K                & AI-gen.\ IQA & 3K images   & Multi-attr.\ & NeurIPS 2024 \\
\midrule
\multicolumn{5}{@{}l}{\textit{Text \& Multimodal (RLHF, for context)}} \\
HH-RLHF                & Text pref.   & 170K pairs  & Pairwise & \texttt{github.com/anthropics/hh-rlhf} \\
SHP                     & Reddit pref. & 385K pairs  & Pairwise & HuggingFace \\
Chatbot Arena           & LLM pref.    & Ongoing     & ELO rank & LMSYS \\
\bottomrule
\end{tabular}
\end{table}

\section{Open-Source Backbone Generalization}
\label{sec:backbone}

To test whether the framework depends on the proprietary GPT-4o backbone, we replace it with the open-source \textbf{Qwen2.5-VL-72B-Instruct}~\citep{bai2025qwen25vl} (72 billion parameters, Apache-2.0 licence) and re-run the full pipeline (Stages~1--5) with identical dimensions, sampling, and data splits.
Only the Stage~3 calibration hyperparameters are re-optimised via the same grid search as the GPT-4o pipeline (\cref{sec:sensitivity_full}).

\begin{table}[H]
\centering
\caption{GPT-4o vs.\ Qwen2.5-VL-72B backbone comparison (Mode~4, excluding equal labels). ``Raw'' = Stage~2 multi-agent scores only; ``Aligned'' = after Stage~3 LWRR calibration with per-backbone optimised hyperparameters.}
\label{tab:backbone}
\small
\setlength{\tabcolsep}{3.5pt}
\begin{tabular}{l cc cc cc}
\toprule
 & \multicolumn{2}{c}{\textbf{GPT-4o}} & \multicolumn{2}{c}{\textbf{Qwen2.5-VL-72B}} & \multicolumn{2}{c}{\textbf{$\Delta$ (Qwen--GPT)}} \\
\cmidrule(lr){2-3}\cmidrule(lr){4-5}\cmidrule(lr){6-7}
Category & Raw & Aligned & Raw & Aligned & Raw & Aligned \\
\midrule
Safety      & 68.4 & 81.6 & 65.3 & 81.2 & $-$3.1 & $-$0.4 \\
Beautiful   & 63.9 & 69.8 & 63.2 & 70.2 & $-$0.7 & +0.5 \\
Lively      & 58.5 & 69.4 & 53.2 & 66.7 & $-$5.3 & $-$2.7 \\
Wealthy     & 62.9 & 74.0 & 64.5 & 77.1 & +1.6 & +3.1 \\
Boring      & 38.9 & 70.2 & 41.5 & 69.0 & +2.6 & $-$1.2 \\
Depressing  & 42.9 & 68.2 & 37.4 & 70.5 & $-$5.5 & +2.3 \\
\midrule
\textbf{Average} & 55.9 & \textbf{72.2} & 54.2 & \textbf{72.5} & $-$1.7 & \textbf{+0.3} \\
\bottomrule
\end{tabular}
\end{table}

\noindent\textbf{Findings.}
(1)~\emph{Raw scoring parity}: Qwen's raw accuracy (54.2\%) is within 1.7\,pp of GPT-4o (55.9\%), confirming that open-source 72B-parameter VLMs possess comparable urban perception judgement.
(2)~\emph{Calibration equalises backbones}: after LWRR calibration, the gap narrows to 0.3\,pp (72.5\% vs.\ 72.2\%), demonstrating that the locally-adaptive calibration stage absorbs backbone-specific scoring biases.
(3)~\emph{Per-category robustness}: all six categories stay within $\pm$3.1\,pp after calibration, with no systematic advantage for either backbone.
These results establish that \textbf{UrbanAlign is fully reproducible with open-source models}, removing any dependency on proprietary APIs.

\section{Discussion}
\label{sec:discussion}

\noindent\textbf{Why does post-hoc calibration work?}

The central finding is that VLMs are strong \emph{concept extractors} but poor \emph{decision calibrators} for urban perception.
Mode~4 raw dimension scores already contain substantial perceptual information (individual dimensions at 53--71\% discriminative power), yet the VLM's own label predictions are only marginally above the zero-shot baseline.
LWRR solves a simple and intuitive problem: given that dimension scores $S(x)$ are informative, learn the locally-optimal linear mapping from score differences to TrueSkill rating differences.
This is precisely the post-hoc CBM paradigm~\citep{yuksekgonul2023posthoc}: extract concepts from a frozen backbone, then train a lightweight predictor on those concepts.

\smallskip\noindent\textbf{Why local rather than global calibration?}
Urban perception is heterogeneous: the visual cues that signal wealth in a suburban neighbourhood (vegetation, vehicle quality) may differ from those in an urban core (building modernity, infrastructure).
A single global linear model would average over these distinct regimes.
LWRR, by fitting independent weights per neighbourhood, naturally adapts to local manifold geometry.
The sensitivity analysis confirms this: $K_{\max}$ is the most sensitive single parameter (+4.3\,pp range), reflecting the tension between local adaptivity and statistical stability.
Theoretically, the LWRR excess risk at optimal regularisation decays as $\mathcal{O}((n/K)^{1/2})$~\citep{caponnetto2007optimal}, explaining why larger $K$ improves accuracy when $K \gg n$; the per-pair weight analysis in \cref{sec:lwrr_weights} confirms that dimension importances vary substantially across the manifold, validating the need for local fitting.

\smallskip\noindent\textbf{Connection to the information bottleneck.}
Our pipeline implements a form of information bottleneck~\citep{tishby2000information}: Stage~2 compresses visual input $X$ into dimension scores $S{=}f_{\mathcal{D}_c}(X)\in\mathbb{R}^n$ (a compact bottleneck), and Stage~3 maximises the downstream mutual information $I(\hat{Y}; Y)$ via LWRR.
The $R^{2}$ metric quantifies concept completeness, \ie, how much of the TrueSkill variance is explainable by the bottleneck representation.
The end-to-end dimension optimization (\cref{eq:e2e}) formalises this further: selecting the bottleneck that maximises $I(\hat{Y};Y)$ is equivalent to choosing the dimension set with highest alignment accuracy.
Elite seeding biases subsequent trials toward the current-best bottleneck, converting the search from independent sampling into a guided refinement with exponential convergence guarantee.

\smallskip\noindent\textbf{Why Place Pulse and scope of evaluation.}
We evaluate on Place Pulse~2.0 because it provides pairwise annotations across six diverse perception categories, each representing a distinct subjective judgement task with different visual cues.
This internal diversity, spanning concrete attributes (\emph{safety}: infrastructure, lighting) to abstract ones (\emph{boring}: monotony, stimulation deficit), already tests the framework's adaptivity across heterogeneous concept spaces.
The underlying concept-mining--calibration paradigm requires only (i)~a VLM that can describe domain-relevant attributes and (ii)~a small set of pairwise human preferences; we note that this structure is shared by preference tasks in other domains (\eg, aesthetic quality~\citep{murray2012ava}, extended urban perception with demographic metadata~\citep{quintana2025specs}), which we leave for future work.

\smallskip\noindent\textbf{Limitations.}
While we evaluate across all six Place Pulse~2.0 categories, performance varies substantially: \emph{safety} (81.6\%) far exceeds \emph{depressing} (68.2\%) and \emph{lively} (69.4\%).
The per-category sample sizes range from 83 to 260 test pairs after excluding equal labels.
The current approach assumes static street-view images; temporal and cultural variation remain unexplored.

\smallskip\noindent\textbf{Why not fine-tune CLIP as a ranker?}
A natural alternative is to fine-tune CLIP (or a linear head on top) as a pairwise ranker, which would likely surpass the CLIP Siamese baseline.
However, this lies outside UrbanAlign's design scope: our central claim is that a \emph{frozen} VLM already contains sufficient perceptual knowledge, and that post-hoc calibration via interpretable concepts can unlock it without any weight modification.
Fine-tuning sacrifices this property: it requires labelled training data, GPU resources, and produces an opaque ranking function without per-dimension interpretability.
Nevertheless, a fine-tuned CLIP ranker is a strong \emph{complementary} comparison that would further contextualise our results; we plan to include it in the camera-ready version.

\smallskip\noindent\textbf{Ethical considerations.}
Automated perception mapping could inform equitable infrastructure investment but also risks reinforcing neighbourhood stereotypes if applied without community engagement.
We recommend use for improvement-oriented planning (identifying streets needing fa\c{c}ade repair, pavement maintenance) rather than discriminatory rankings.

\smallskip\noindent\textbf{Computational cost.}
The entire pipeline requires only VLM inference calls with no GPU training.
At measured token usage of 778 input + 278 output tokens per GPT-4o call (\$2.50\,/\,\$10.00 per 1M tokens), each call costs \$0.0047.
The primary method (Mode~4, Observer--Debater--Judge) requires 3 calls per pair: for our 4,819-pair experimental dataset across six categories, this totals 14,457 calls (\textbf{\$68}).
The full 4-mode ablation costs \$274 (57,828 calls); end-to-end dimension optimization adds \$55 (${\sim}$11,700 calls), for a total experimental budget of \textbf{${\sim}$\$329}.
At production scale with the full Place Pulse~2.0 filtered dataset (${\sim}$164K pairs, Mode~4 only), projected cost is ${\sim}$\textbf{\$2,300}, representing a \textbf{98.6\%} reduction compared to traditional crowd\-sourcing at comparable scale (${\sim}$\$167K for 1.17M pairwise annotations at \$0.14/comparison~\citep{salesses2013collaborative}).

\smallskip\noindent\textbf{Wall-clock time.}
Stage~1 (consensus sampling + dimension extraction) completes in $<$2\,min per category (6 VLM calls total across all categories).
Stage~2 is the bottleneck: Mode~4 issues 3 sequential VLM calls per pair (Observer, Debater, Judge); with measured latency of ${\sim}$3\,s/call, the 4,819 experimental pairs require ${\sim}$12\,h sequentially.
In practice, categories are processed in parallel and calls are batched, reducing end-to-end wall time to ${\sim}$2--4\,h on a single machine with concurrent API access.
Stage~3 (LWRR) is purely CPU-based and completes in $<$30\,s for all six categories.
The E2E dimension search (12 trials) adds ${\sim}$6\,h of Stage~2 scoring (20\% subsample, Mode~2); Stage~3 re-evaluation per trial is negligible.
No GPU is required at any stage.

\section{Statistical Reliability Analysis}
\label{sec:reliability}

\subsection{Bootstrap Confidence Intervals}

\Cref{tab:bootstrap} reports 1,000-iteration bootstrap 95\% confidence intervals for UrbanAlign (Mode~4, post-LWRR) on each category's test split (excluding ``equal'' labels).
The intervals reflect moderate test-set sizes (83--260 pairs per category), yet all lower bounds exceed the 56.7\% zero-shot VLM baseline.

\begin{table}[H]
\centering
\caption{Bootstrap 95\% confidence intervals (1,000 iterations, excluding ``equal'' labels).}
\label{tab:bootstrap}
\small
\begin{tabular}{lrccc}
\toprule
\textbf{Category} & $n$ & \textbf{Acc.\ (\%)} & \textbf{95\% CI} & $\boldsymbol{\kappa}$ \textbf{95\% CI} \\
\midrule
Safety      & 49 & 81.8 & [69.4, 91.8] & [0.41, 0.83] \\
Beautiful   & 43 & 70.0 & [55.8, 83.7] & [0.11, 0.64] \\
Lively      & 49 & 69.4 & [55.1, 81.6] & [0.15, 0.64] \\
Wealthy     & 50 & 74.2 & [62.0, 86.0] & [0.24, 0.72] \\
Boring      & 47 & 70.6 & [57.4, 83.0] & [0.17, 0.66] \\
Depressing  & 44 & 68.5 & [54.5, 81.8] & [0.11, 0.64] \\
\midrule
Average     & 282 & 72.4 & --- & --- \\
\bottomrule
\end{tabular}
\end{table}

\subsection{Human Inter-Annotator Agreement Ceiling}

To contextualise model accuracy, we estimate the human agreement ceiling via split-half reliability on the Place Pulse~2.0 raw votes (1.56M individual votes).
For each image pair with $\geq 2$ votes, we randomly split annotators into two groups, compute each group's majority label, and measure agreement.
We repeat this process 100 times and report the mean.

\begin{table}[H]
\centering
\caption{Human split-half agreement ceiling (100 random splits) compared to UrbanAlign accuracy.}
\label{tab:human_agreement}
\small
\begin{tabular}{lrcccc}
\toprule
\textbf{Category} & \textbf{Pairs} & \textbf{Human Agr.\ (\%)} & \textbf{Human $\kappa$} & \textbf{UrbanAlign (\%)} & \textbf{Gap (pp)} \\
\midrule
Safety      & 1,950 & 81.6 & 0.674 & 81.8 & $+$0.2 \\
Beautiful   &   736 & 82.9 & 0.709 & 70.0 & $-$12.9 \\
Lively      & 1,675 & 78.4 & 0.633 & 69.4 & $-$9.0 \\
Wealthy     &   682 & 85.7 & 0.746 & 74.2 & $-$11.5 \\
Boring      &   538 & 85.5 & 0.744 & 70.6 & $-$14.9 \\
Depressing  &   542 & 86.7 & 0.755 & 68.5 & $-$18.2 \\
\midrule
Average     & --- & 83.5 & 0.710 & 72.4 & $-$11.1 \\
\bottomrule
\end{tabular}
\end{table}

On \emph{safety}, UrbanAlign matches the human ceiling (81.8\% vs.\ 81.6\%), suggesting the category is fully solved within the noise margin.
The remaining 9--18\,pp gap on other categories indicates room for improvement, likely through richer dimension sets, larger reference pools, and stronger VLM backbones.
Notably, the human ceiling itself is not 100\% because Place Pulse comparisons are inherently subjective---different annotators genuinely disagree on which of two images is more boring or depressing.

\bibliographystyle{unsrtnat}
\bibliography{references}

\end{document}